\documentclass{article} %
\usepackage{colm2024_conference}

\usepackage{booktabs}
\usepackage{graphicx}
\usepackage{enumitem}
\usepackage{wrapfig}
\usepackage{algorithm}
\usepackage{algpseudocode}
\usepackage{wrapfig}
\usepackage{float}
\usepackage{microtype}
\usepackage{amsmath}
\usepackage{amssymb}
\usepackage{colortbl}
\usepackage[utf8]{inputenc}
\usepackage{caption}
\usepackage{subcaption}
\usepackage{xcolor}
\usepackage{setspace}
\usepackage{url}
\usepackage{multirow}
\usepackage{colortbl}
\usepackage{tabularx}
\usepackage{blindtext}
\usepackage{pgfplots}
\pgfplotsset{compat=1.18}
\usepackage{tikz}
\usetikzlibrary{er, positioning, bayesnet}
\usepackage{makecell}
\usepackage{tipa}
\usepackage{siunitx}
\usepackage{nicefrac}
\usepackage{tocloft}
\usepackage{listings}
\usepackage[raster, skins]{tcolorbox} %
\usepackage{xltabular}
\usepackage{adjustbox}
\usepackage{xurl}
\usepackage{longtable}  

\usepackage{amsmath,amsfonts,bm}

\def\eqref#1{equation~\ref{#1}}

\def\1{\bm{1}}

\DeclareMathAlphabet{\mathsfit}{\encodingdefault}{\sfdefault}{m}{sl}
\SetMathAlphabet{\mathsfit}{bold}{\encodingdefault}{\sfdefault}{bx}{n}

\definecolor{lightgray}{rgb}{0.9,0.9,0.9}

\usepackage{afterpage}
\usepackage{cuted}
\usepackage{stfloats}
\usepackage{ulem}           
\usepackage{soul}           
\usepackage{tablefootnote}  
\usepackage{comment}        

\usepackage{pifont}         
\usepackage{wasysym}

\usepackage{longtable}
\usepackage{array}
\usepackage{caption}  
\usepackage{ragged2e} 

\usepackage{listings}
\usepackage{xcolor}
\usepackage{graphicx}
\usepackage{wrapfig}
\usepackage{amsthm}

\newtheorem{theorem}{Theorem}[section]

\newtheorem{corollary}[theorem]{Corollary}

\definecolor{lightblue}{RGB}{173,216,230}
\definecolor{lightgreen}{RGB}{144,238,144}
\definecolor{lightpink}{RGB}{255, 228, 225}
\definecolor{lightred}{RGB}{255,182,193}
\definecolor{lightyellow}{RGB}{255,255,224}
\definecolor{lightpurple}{RGB}{221,160,221}
\definecolor{lightgray}{RGB}{211,211,211}
\definecolor{lightorange}{RGB}{255,218,185}
\definecolor{lightpeach}{rgb}{1.0, 0.882, 0.788}
\definecolor{lightcyan}{rgb}{0.8196, 0.9725, 0.9804}
\definecolor{sh_blue}{rgb}{0,0.60,0.93}
\definecolor{sh_red}{rgb}{0.8627, 0.3098, 0.3176}
\definecolor{highlight}{RGB}{255,255,0}
\definecolor{warning}{RGB}{255,99,71}
\definecolor{success}{RGB}{50,205,50}
\definecolor{info}{RGB}{30,144,255}
\definecolor{top1}{RGB}{255,179,179}
\definecolor{top2}{RGB}{255,217,179}
\definecolor{top3}{RGB}{255,255,179}
\definecolor{textblue}{RGB}{94,159,220} 
\definecolor{textgreen}{RGB}{59,125,35} 
\definecolor{textorange}{RGB}{192,80,21} 
\definecolor{tagred}{RGB}{196,15,15} 
\definecolor{tagblue}{RGB}{33,95,154} 
\definecolor{teaserblue}{RGB}{33,95,154} 
\definecolor{teasergree}{RGB}{57,158,163} 
\definecolor{teaserpurpe}{RGB}{105,111,173} 

\definecolor{primary}{RGB}{70,130,180}
\definecolor{secondary}{RGB}{119,136,153}
\definecolor{accent}{RGB}{255,140,0}

\definecolor{customblue}{HTML}{E7EFFA}
\definecolor{custompink}{HTML}{F7E1ED}

\renewcommand{\arraystretch}{1.25}

\makeatletter
\DeclareRobustCommand\onedot{\futurelet\@let@token\@onedot}
\def\@onedot{\ifx\@let@token.\else.\null\fi\xspace}

\makeatother

\title{
VisionCreator: A Native Visual-Generation Agentic Model \\ with Understanding, Thinking, Planning and Creation}

\author{%
Jinxiang Lai$^{2*}$, Zexin Lu$^{1*}$, Jiajun He$^{1}$, Rongwei Quan$^{1}$, Wenzhe Zhao$^{1}$, Qinyu Yang$^{1}$, Qi Chen$^{1}$, Qin Lin$^{1\dagger}$, Chuyue Li$^{1}$, Tao Gao$^{1}$, Yuhao Shan$^{1}$, Shuai Shao$^{1}$, \\  Song Guo$^{2\S}$, Qinglin Lu$^{1\dagger}$\\
{\small{$^1$Tencent Hunyuan,
$^2$Hong Kong University of Science and Technology}}\\
{\small{$^*${Equal contribution}, $^\S${Corresponding Author}, $^\dagger${Project lead}}}
}

\begin{document}

\maketitle

\begin{abstract}
Visual content creation tasks demand a nuanced understanding of design conventions and creative workflows-capabilities challenging for general models, while workflow-based agents lack specialized knowledge for autonomous creative planning. To overcome these challenges, we propose \textbf{VisionCreator}, a native visual-generation agentic model that unifies Understanding, Thinking, Planning, and Creation (UTPC) capabilities within an end-to-end learnable framework. Our work introduces four key contributions: (i) \textbf{VisGenData-4k} and its construction methodology using metacognition-based \textbf{VisionAgent} to generate high-quality creation trajectories with explicit UTPC structures; (ii) The VisionCreator agentic model, optimized through \textbf{Progressive Specialization Training (PST)} and \textbf{Virtual Reinforcement Learning (VRL)} within a high-fidelity simulated environment, enabling stable and efficient acquisition of UTPC capabilities for complex creation tasks; (iii) \textbf{VisGenBench}, a comprehensive benchmark featuring 1.2k test samples across diverse scenarios for standardized evaluation of multi-step visual creation capabilities; (iv) Remarkably, our VisionCreator-8B/32B models demonstrate superior performance over larger closed-source models across multiple evaluation dimensions. Overall, this work provides a foundation for future research in visual-generation agentic systems.
\end{abstract}    
\section{Introduction}
\label{sec:intro}
AI-assisted visual content creation has revolutionized workflows from professional design to social media. The field has evolved from single-image generation \citep{ho2020denoising,ramesh2022hierarchical,flux2024,saharia2022photorealistic,lin2024open} to complex multi-modal synthesis \citep{wu2025automated,xiao2025captain,xu2025mm,xue2025comfybench,guo2025comfymind}, demanding systems that can understand creative intent, plan multi-step operations, and autonomously execute intricate workflows.
As shown in Fig.\ref{fig:vision_creator}, current approaches to autonomous visual creation can be categorized into three main paradigms, each with distinct limitations: 
(a) \textit{General-purpose Unified Multimodal Models (UMM)} \cite{deng2025emerging,li2025omniflow,lai2024spider,cui2025emu3} leverage large-scale pre-training to achieve impressive visual understanding, but lack the domain-specific knowledge required for autonomous creative planning and struggle to decompose complex objectives without extensive prompt engineering. 
(b) \textit{Workflow-specific Agent} \cite{wu2025automated,xiao2025captain,xu2025mm} employ predefined pipelines for specific domains like movie generation or story creation, but their rigid architectures cannot adapt to diverse creative tasks or handle unexpected outcomes during execution. 
(c) \textit{Workflow-guided Agent} \cite{xue2025comfybench,guo2025comfymind} orchestrate external tools through carefully designed prompts and coordination logic, leveraging general language models to interpret requests and sequence operations.
However, this approach faces several limitations: (i) Reliance on prompt engineering rather than learned domain knowledge, limiting creative understanding ; (ii) Explicitly programmed coordination logic that restricts adaptability to diverse tasks; and (iii) Inability to be jointly optimized end-to-end for creative task performance.

\begin{figure*}[t]
\vspace{-1mm}
\centering
\includegraphics[width=0.99\textwidth]{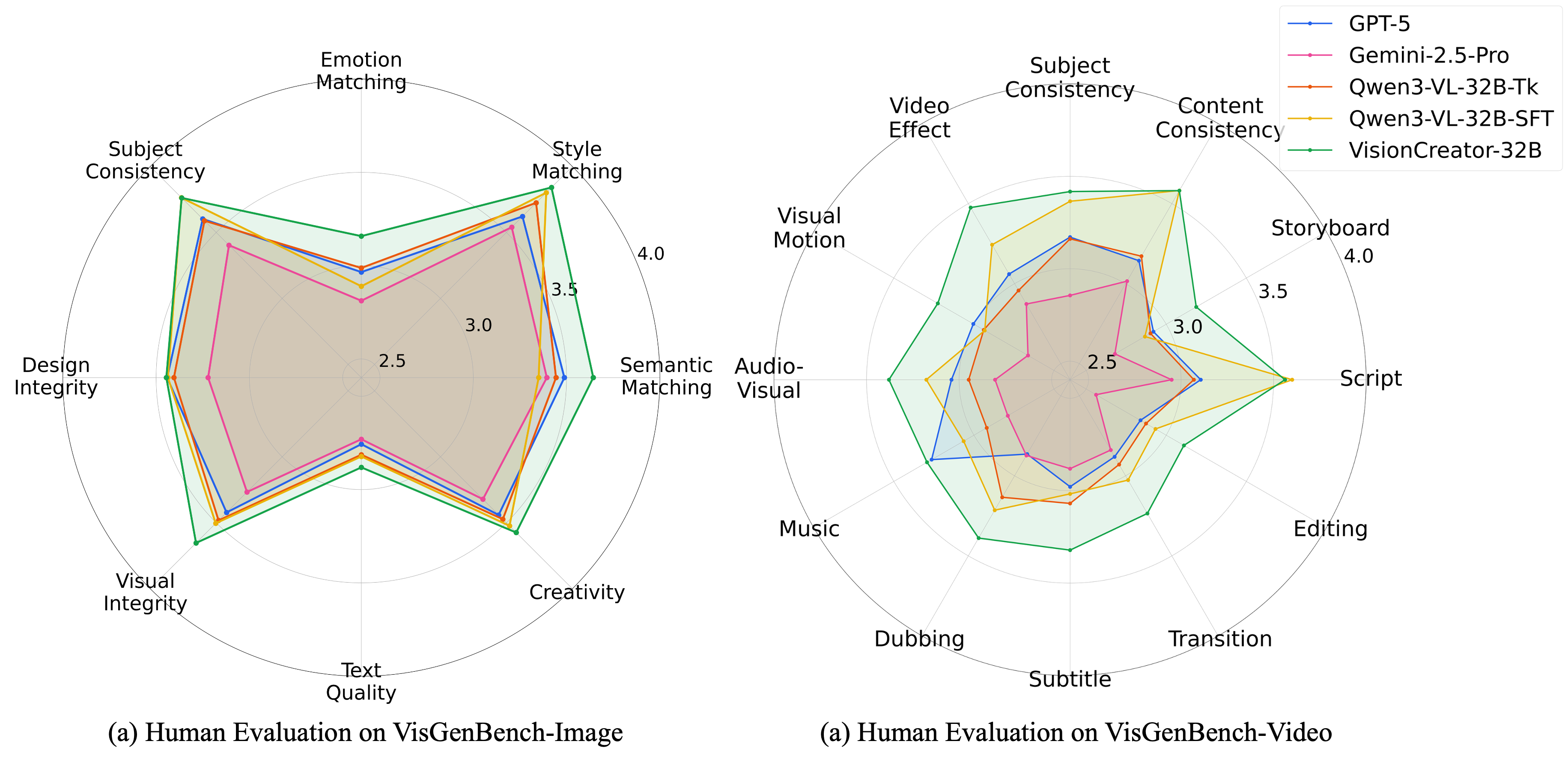}
\vspace{-2mm}
\caption{Human Evaluation results on VisGenBench-Image and VisGenBench-Video.}
\vspace{-0mm}
\label{fig:radar}
\end{figure*}

To overcome these limitations, we propose \textbf{VisionCreator}, a \textit{native} visual-generation agentic model that unifies Understanding, Thinking, Planning, and Creation (\textbf{UTPC}) capabilities in an end-to-end learnable framework, as shown in Fig.\ref{fig:vision_creator} (d). Unlike existing approaches that rely on predefined workflows or external template workflows, our native architecture intrinsically integrates the capabilities of \textit{Understanding} design conventions and user intent, \textit{Thinking} through complex creative constraints, \textit{Planning} multi-step execution trajectories, and \textit{Creation} of high-quality and diverse visual creation tasks.
However, realizing this new paradigm faces several critical challenges:

\noindent\ding{172} \textbf{Data Bottleneck}: Currently, no comprehensive datasets exist for training agents to perform visual content creation through tool invocation. The lack of high-quality trajectories prevents supervised learning of the UTPC capabilities.

\noindent\ding{173} \textbf{Task Complexity}: How to develop models that can handle the full spectrum of visual creation challenges, which encompass (i) Diverse task types, (ii) Varying difficulty levels from basic generation to advanced composition, and (iii) Complex creation tasks requiring 20+ execution steps? Existing approaches face significant limitations: specialized systems excel in narrow domains but fail to generalize across diverse tasks, while general models lack the depth for sophisticated creative reasoning and struggle with long-horizon consistency and adaptive strategy adjustment.

\noindent\ding{174} \textbf{Training Difficulty}: How to establish an effective and efficient training paradigm for such a native agent? The conventional SFT+RL framework faces significant obstacles: (i) SFT phase struggles to balance general capability preservation with domain-specific specialization, often leading to catastrophic forgetting or insufficient expertise; (ii) Direct online RL training with real tools incurs prohibitive costs and instability due to expensive API invocation and limited concurrency. Furthermore, designing accurate reward signals for multi-step creative trajectories is particularly challenging, as imperfect reward functions are highly vulnerable to reward hacking.

\begin{figure*}[t]
\vspace{-0mm}
\centering
\includegraphics[width=0.88\textwidth]{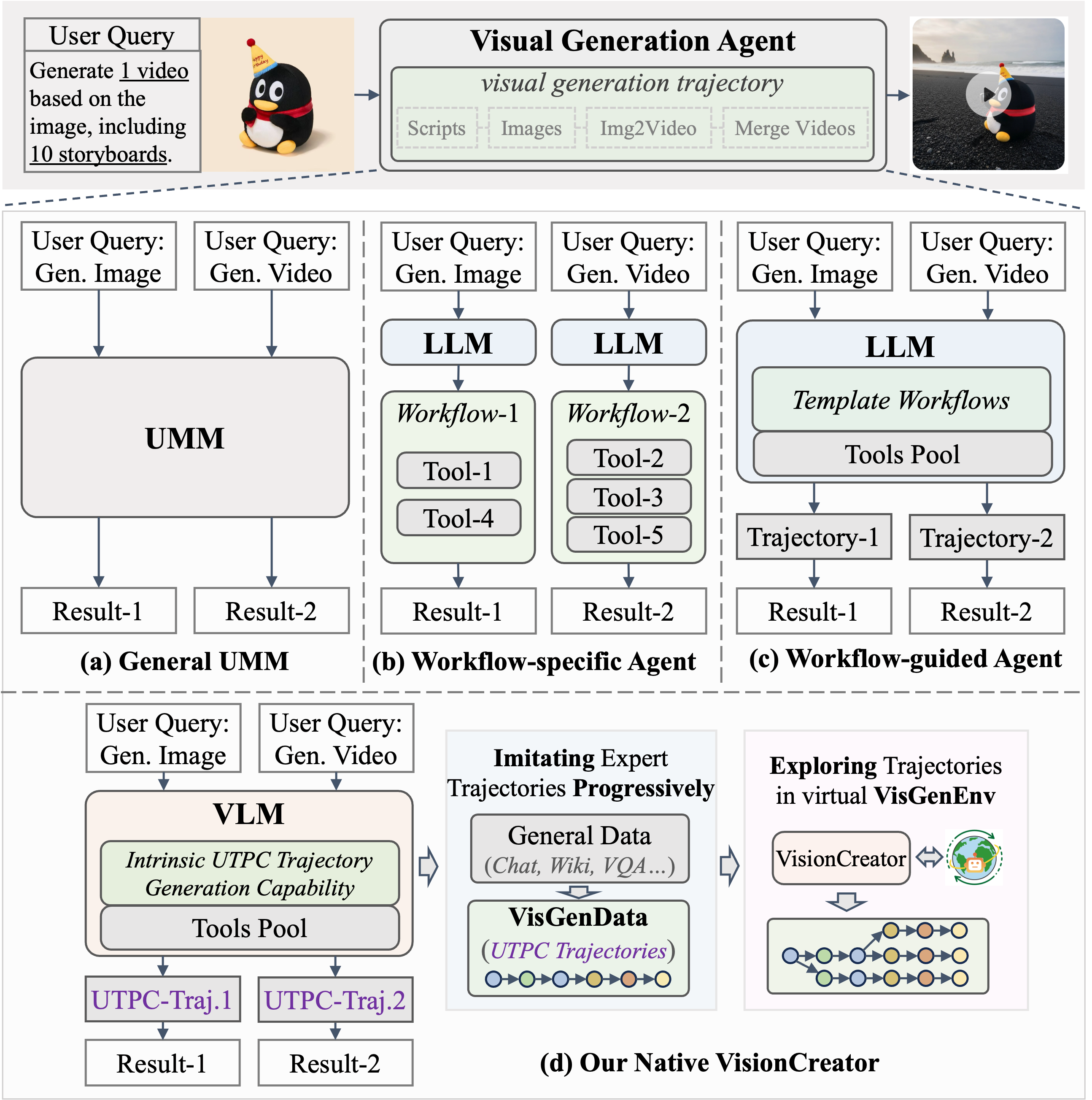}
\vspace{-1mm}
\caption{Framework comparisons. (a) UMM. (b) Workflow-specific Agent. (c) Workflow-guided Agent. (d) Our Native VisionCreator.}
\vspace{-1mm}
\label{fig:vision_creator}
\end{figure*}

To address these challenges, we propose:

\noindent(i) \textbf{VisGenData-4k with UTPC Structure}: We design a metacognition-based \textit{VisionAgent} to generate a comprehensive dataset following the UTPC structure, featuring diverse visual creation tasks across multiple difficulty levels. 
Through rigorous human quality inspection, we meticulously filter and retain only the highest-quality data samples. The resulting VisGenData-4k provides diverse and high-quality execution trajectories that explicitly capture \textit{Understanding} of design conventions, \textit{Thinking} through creative constraints, \textit{Planning} of multi-step trajectories, and \textit{Creation} of visual content, offering rich supervision signals for complex creative workflows.

\noindent(ii) \textbf{Progressive Specialization Training (PST)}: We introduce a novel Progressive Specialization Training methodology that cultivates UTPC capabilities through two-stage optimization. PST effectively addresses the generalization-specialization trade-off by first establishing robust \textit{Understanding} and \textit{Thinking} capacities through general foundation learning, followed by targeted domain specialization to enhance \textit{Planning} and \textit{Creation} expertise. This progressive strategy not only prevents catastrophic forgetting of general abilities but also efficiently identifies optimal data composition for stagewise specialization, enabling the model to develop comprehensive UTPC capacities while maintaining strong cross-domain reasoning abilities.

\noindent(iii) \textbf{Virtual VisGenEnv Construction}: We construct {VisGenEnv}, a virtual environment for VRL. It features 36 tools with high-fidelity simulation of their behaviors. Multimodal outputs are simulated by returning random samples from a media database, providing correct physical attributes. This design enables effective learning of workflow planning through accurate tool behavior simulation.

\noindent(iv) \textbf{Virtual Reinforcement Learning (VRL) with LtrReward}: We develop an innovative \textit{Virtual Reinforcement Learning (VRL)} paradigm that conducts the entire reinforcement learning using \textit{Long Trajectory Reasoning Reward (LtrReward)} within the high-fidelity \textit{VisGenEnv}. This approach bypasses the prohibitive cost of thousands of GPUs by leveraging simulated tool-call behaviors and functional logic, enabling stable and scalable learning of high-quality planning and action trajectories.
Moreover, we provide a theoretical analysis that establishes formal guarantees on sim-to-real transfer and real-world performance improvement.

Finally, we introduce \textbf{VisGenBench}, a comprehensive benchmark designed for evaluating \textit{visual generation agentic models} that operate through multi-step tool invocation to accomplish complex image and video creation tasks. Our benchmark encompasses: (i) \textit{Comprehensive Test Suite} - featuring 1.2k test samples including 400 image-generation tasks and 800 video-generation tasks; (ii) \textit{Diverse Applications} - spanning 10 evaluation dimensions across 35+ real-world scenarios; (iii) \textit{Standardized Protocol} - ensuring reproducible evaluation through structured scoring rubrics.

Overall, our contributions are: (i) The VisionCreator, a novel native visual-generation agentic model that unifies UTPC capabilities in an end-to-end learnable framework; (ii) VisGenData-4k and its construction framework using metacognition-based VisionAgent to generate high-quality creation trajectories with UTPC structures; (iii) A progressive training methodology combining Progressive Specialization Training (PST) and Virtual Reinforcement Learning (VRL) with LtrReward, enabling stable and efficient learning of complex creation trajectories entirely within a virtual environment VisGenEnv; (iv) VisGenBench benchmark with 1.2k test samples across diverse scenarios for standardized evaluation of multi-step visual creation capabilities.

\section{Related Works}
\label{sec:related}

\subsection{Image Generation}
Current image generation models primarily fall into two categories: Autoregressive~\citep{chen2020generative,fan2024fluid,han2024infinity,tian2024visual,sun2024autoregressive,pang2024next} models and Diffusion~\citep{ho2020denoising,ramesh2022hierarchical,flux2024,saharia2022photorealistic,lin2024open} models. While these models provide powerful single-step image generation capabilities, they primarily focus on the \textit{Creation} aspect of visual content generation. Our VisionCreator builds upon these fundamental generation technologies but extends them by integrating comprehensive \textit{Understanding}, \textit{Thinking}, and \textit{Planning} capabilities. This allows our agent to not only generate individual images but also reason about complex creative requirements and plan multi-step visual creation workflows that leverage these underlying generation models as tools.

\subsection{Video Generation}
Video generation methods build on image models by adding time-based processing. Approaches like Make-A-Video~\citep{singer2022make} and SVD~\citep{blattmann2023stable} extend image generation to video, while newer architectures like DiT~\citep{peebles2023scalable} and MMDiT~\citep{esser2024scaling} in models such as CogvideoX~\citep{yang2024cogvideox} show progress in handling longer videos. These video generation tools are important for our agent's creation ability, but they work separately. Our VisionCreator connects these tools through planning and reasoning to handle complete creation tasks.

\subsection{Visual Generation Agents}
Current approaches to autonomous visual creation include three main agent paradigms: (i) {Workflow-specific Agents} (e.g., MovieAgent~\citep{wu2025automated}, Captain Cinema~\citep{xiao2025captain}, MM-StoryAgent~\citep{xu2025mm}) employ predefined pipelines for specialized domains but lack adaptability to diverse creative tasks. (ii) {ComfyUI Workflow Generation} methods (e.g., ComfyAgent~\citep{xue2025comfybench}, ComfyMind~\citep{guo2025comfymind}, ComfyUI-R1~\cite{xu2025comfyui}, ComfyGPT~\cite{huang2025comfygpt}) specialize in generating ComfyUI-format workflows, which limits their visual creation capability in general API scenarios. (iii) {Workflow-guided Agents} \citep{xue2025comfybench,guo2025comfymind} orchestrate external tools through prompt engineering but face limitations in creative understanding depth and end-to-end optimization. These limitations motivate our native visual-generation agent that intrinsically integrates UTPC capabilities in an end-to-end learnable framework.

\begin{figure*}[ht]
\vspace{-0mm}
\centering
\includegraphics[width=0.99\textwidth]{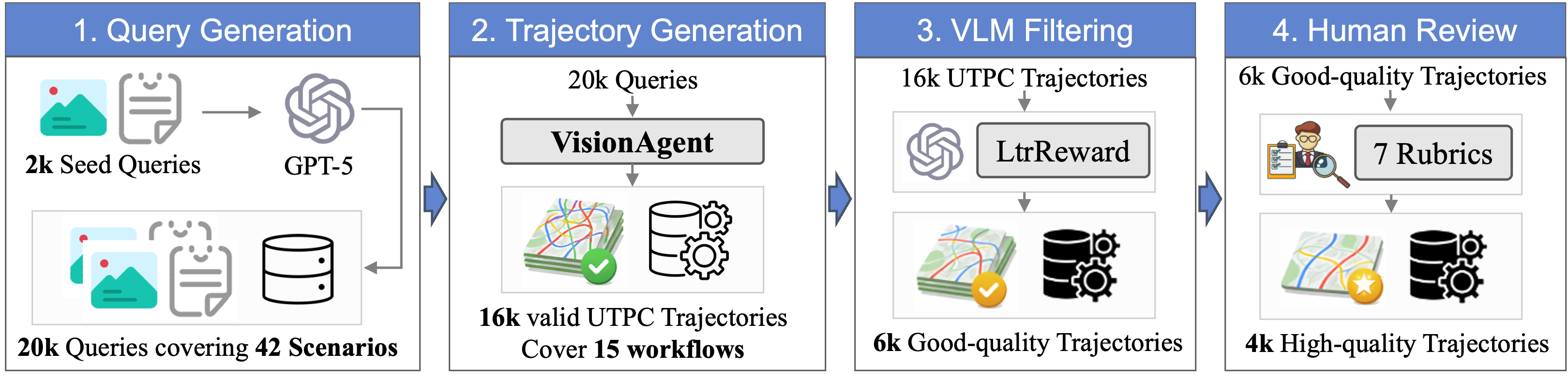}
\vspace{-1mm}
\caption{VisGenData-4k construction pipeline.}
\vspace{-3mm}
\label{fig:data_gen}
\end{figure*}

\section{VisGenData-4k with UTPC Structure}
\label{sec:training_dataset}
To tackle the data bottleneck in training visual creation agents, we design VisionAgent, a dataset generation framework based on a Metacognition paradigm. 
To construct a high-quality VisGenData dataset, VisionAgent employs commercial proprietary models (such as GPT-5, GPT-4o, Veo3, Sora2, etc) for multimodal data generation, and we further filter low-quality trajectories with algorithms and human experts.
As shown in Fig.~\ref{fig:data_gen}, the construction pipeline is as follows: (i) VisionAgent first generates 16k trajectories from 20k queries covering 42 scenarios. (ii) With the rigorous LtrReward and VLM-Grader methods, we remove 10k low-quality trajectories and obtain 6k candidate trajectories. (iii) These subsequently undergo a manual review by human experts, where 2k undesired trajectories are filtered out, resulting in high-quality 4k trajectories.

\begin{wrapfigure}{r}{0.6\textwidth}
\centering
\includegraphics[width=0.9\linewidth]{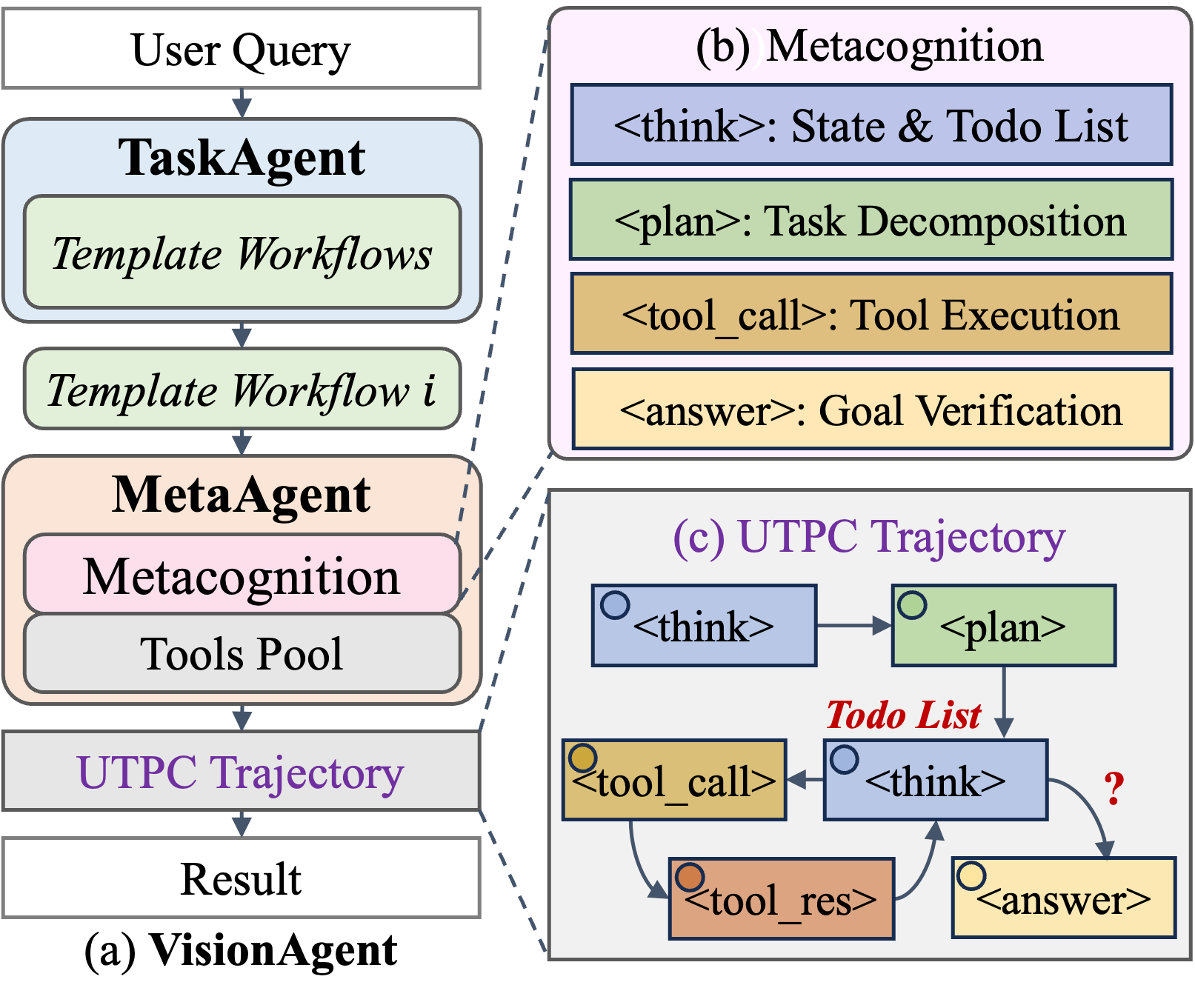}
\vspace{-2mm}
\caption{VisionAgent framework for dataset generation.}
\label{fig:visionagent}
\vspace{-2mm}
\end{wrapfigure}

\subsection{VisionAgent for Dataset Construction}
\label{subsec:dataset_construction}
As shown in Fig.\ref{fig:visionagent}, our VisionAgent generates high-quality execution trajectories that capture the complete reasoning process for complex visual creation tasks. 
VisionAgent with metacognition achieves a 72\% task success rate, representing a 30\% improvement over the baseline method that relies solely on thinking.

\noindent\textbf{Dual-Agent Architecture.} Our framework employs a dual-agent architecture that separates task understanding from execution reasoning:
(1) \textbf{TaskAgent}: Serves as the task classifier and router. It analyzes user inputs and performs fine-grained task classification across the 21 distinct task types, then selects appropriate predefined workflow templates and tools pool for specific task categories.
(2) \textbf{MetaAgent}: Functions as the core reasoning engine with metacognitive capabilities. It receives both the selected workflow and tools pool as inputs, then executes structured reasoning through four standardized reasoning types defined in metacognition.

\begin{wrapfigure}{r}{0.6\textwidth}
    \centering
    \includegraphics[width=0.9\linewidth]{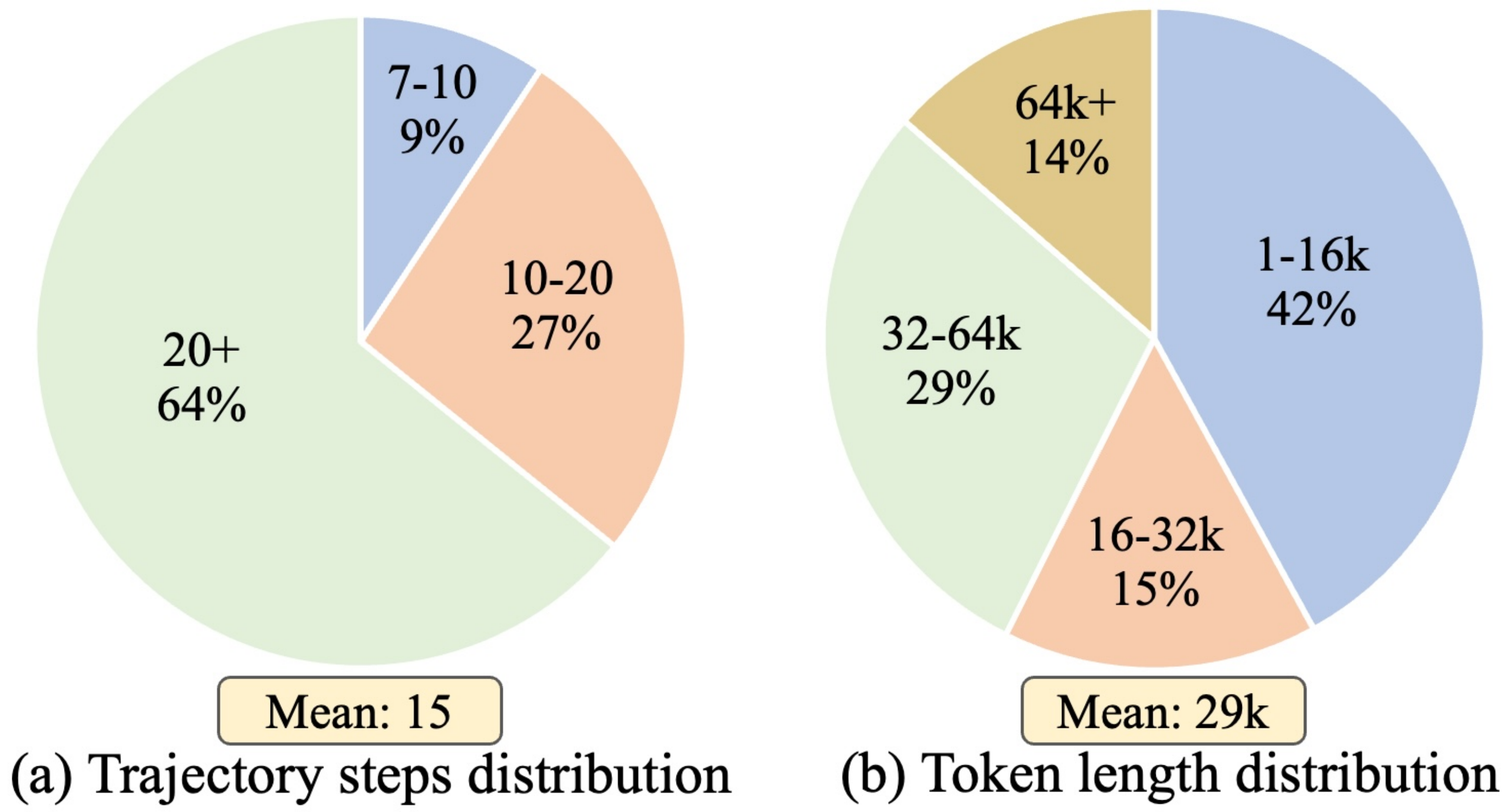}
    \vspace{-2mm}
    \caption{VisGenData-4k dataset statistics.}
    \label{fig:visgendata}
    \vspace{-2mm}
\end{wrapfigure}

\noindent\textbf{Metacognitive Reasoning Process.} The metacognition defines four reasoning types: the \textit{\textless think\textgreater} phase maintains situational awareness and todo-list through continuous state evaluation; the \textit{\textless plan\textgreater} phase constructs executable task sequences by decomposing objectives and dependencies; the \textit{\textless tool\_call\textgreater} phase invokes appropriate tools based on plan blueprints and analytical reasoning; and the \textit{\textless answer\textgreater} phase verifies goal completion, forming a closed-loop execution process.
This Metacognitive reasoning process guides the MetaAgent to generate the UTPC trajectory.

\noindent\textbf{Reference Workflow Integration.} We incorporate 15 predefined workflow templates as best-practice guides, ensuring planning remains flexible yet stays on track. These workflows provide domain-specific execution patterns for various visual creation tasks, representing 15 distinct application scenarios from storyboards to animated short films.

\subsection{Dataset Composition and Statistics}
\label{subsec:dataset_composition}
Fig.\ref{fig:visgendata} shows the statistics of our VisGenData-4k, which exhibits the following key features:
(1) \textbf{Diverse Task Types}: Encompassing 21 distinct task types (including storyboards, marketing posters, product marketing videos, animated short films, etc), this diversity is crucial for training agents to handle a broad range of real-world creative demands, significantly enhancing their adaptability and practical applicability.
(2) \textbf{Complex Trajectory Structure}: With a mean of 15 steps and 64\% of trajectories exceeding 20 steps, this complexity is crucial for training agents to decompose and plan long-horizon tasks, fostering robust problem-solving capabilities in visual creation.
(3) \textbf{Rich Contextual Information}: The substantial token length (mean: 29k, 43\% over 32k) equips agents with the ability to process and utilize extensive contextual cues, significantly enhancing their capacity for detailed and context-aware generation.

\section{Agentic Post-Training}
\label{sec:training}

\subsection{Agentic Framework}
\label{sec:framework}
As shown in Fig.~\ref{fig:framework}, VisionCreator is formulated as a unified agent that integrates \textit{Understanding, Thinking, Planning, and Creation} (UTPC) capabilities to accomplish complex visual generation tasks.
Formally, we model the agent as a policy $\pi_\theta$ operating over long-horizon multimodal trajectories:
$\tau = (o_0, a_0, o_1, a_1, \dots, o_T)$,
where $o_t$ denotes multimodal observations (textual instructions, intermediate tool feedback, and virtual visual states), and $a_t$ denotes agent actions including reasoning tokens, planning steps, and tool invocations.
The training process follows a two-stage agentic post-training paradigm:
(1) \textbf{Progressive Specialization Training (PST)}, which initializes a strong policy prior via supervised learning over expert UTPC trajectories.
(2) \textbf{Virtual Reinforcement Learning (VRL)}, which further optimizes long-horizon planning and tool-use strategies through large-scale exploration in a simulated environment.

\subsection{Progressive Specialization Training}
\label{subsec:pst}

The goal of Progressive Specialization Training (PST) is to learn an
initial policy $\pi_{\theta_0}$ that simultaneously preserves
\textit{general reasoning competence} while acquiring
\textit{domain-specific visual creation ability},
thereby enabling a functional visual content creation agent
rather than a narrowly tuned generator.
Let the supervised dataset be
$\mathcal{D} = \mathcal{D}_{\text{gen}} \cup \mathcal{D}_{\text{vis}}$,
where $\mathcal{D}_{\text{gen}}$ contains large-scale general reasoning
and tool-use trajectories, and $\mathcal{D}_{\text{vis}}$ contains
expert-curated visual creation trajectories (VisGenData-4k).
Standard supervised fine-tuning (SFT) minimizes
\begin{equation}
\mathcal{L}_{\text{SFT}}(\theta)
=
\mathbb{E}_{(o,a)\sim\mathcal{D}}
\left[-\log \pi_\theta(a \mid o)\right].
\end{equation}

However, naive single-stage SFT exhibits two fundamental failure modes.
Training only on $\mathcal{D}_{\text{vis}}$ leads to catastrophic
forgetting of general reasoning and planning ability,
resulting in nearly zero agent competence; empirically,
Tab.~\ref{tab:ablation_study} shows performance dropping to
\textbf{0.007}, indicating the model is unable to function as a
visual creation agent.
Conversely, one-stage mixed SFT on
$\mathcal{D}_{\text{gen}} \cup \mathcal{D}_{\text{vis}}$
avoids catastrophic forgetting but yields
\textit{suboptimal specialization}, since the dominance of
$\mathcal{D}_{\text{gen}}$ suppresses learning of visual-creation
behaviors and degrades downstream agent performance.
These observations reveal a necessary condition for visual agents:
\[
\textbf{General Competence Preservation}
+
\textbf{Strong Visual Agent Specialization},
\]
which neither naive SFT strategies can satisfy simultaneously.

PST resolves this conflict through a controlled two-stage curriculum
that induces a gradual distribution shift.
In Stage~1 (general foundation learning),
\begin{equation}
\mathcal{D}^{(1)}
=
\mathcal{D}_{\text{gen}}^{500\text{K}}
\cup
\lambda\,\mathcal{D}_{\text{vis}},
\end{equation}
establishing robust reasoning, planning, and tool-use capabilities
while lightly anchoring the policy to the visual generation agent domain.
In Stage~2 (targeted specialization),
\begin{equation}
\mathcal{D}^{(2)}
=
\mathcal{D}_{\text{gen}}^{200\text{K}}
\cup
\lambda\,\mathcal{D}_{\text{vis}},
\end{equation}
the increased effective influence of $\mathcal{D}_{\text{vis}}$
drives specialization toward visual content creation,
while continued exposure to $\mathcal{D}_{\text{gen}}$
prevents catastrophic forgetting.
Overall, PST learns a structured initialization
\begin{equation}
\pi_{\theta_0}
\approx
\arg\min_\theta
\mathbb{E}_{(o,a)\sim
\mathcal{D}^{(1)} \rightarrow \mathcal{D}^{(2)}}
\left[-\log \pi_\theta(a \mid o)\right].
\end{equation}
which constrains downstream reinforcement learning (RL)
to a policy region that already satisfies both
general competence and visual specialization.
Experimental results further validate the necessity of PST.
Compared with one-stage SFT, PST achieves
substantially stronger performance on visual creation agent tasks,
demonstrating that progressive specialization is essential for
learning effective UTPC behaviors.
Moreover, PST provides a significantly better initialization for RL:
the initial reward score before RL training increases from
\textbf{0.64} (one-stage SFT) to \textbf{0.87} (PST),
a gain of \textbf{+0.23}.
This improved starting point directly translates into
optimization efficiency—RL convergence is accelerated by
approximately \textbf{50\%}.
These findings confirm that PST not only improves
final agent capability, but also fundamentally reduces
the difficulty of downstream reinforcement learning.

\begin{figure*}[t]
\vspace{-2mm}
\centering
\includegraphics[width=1.0\textwidth]{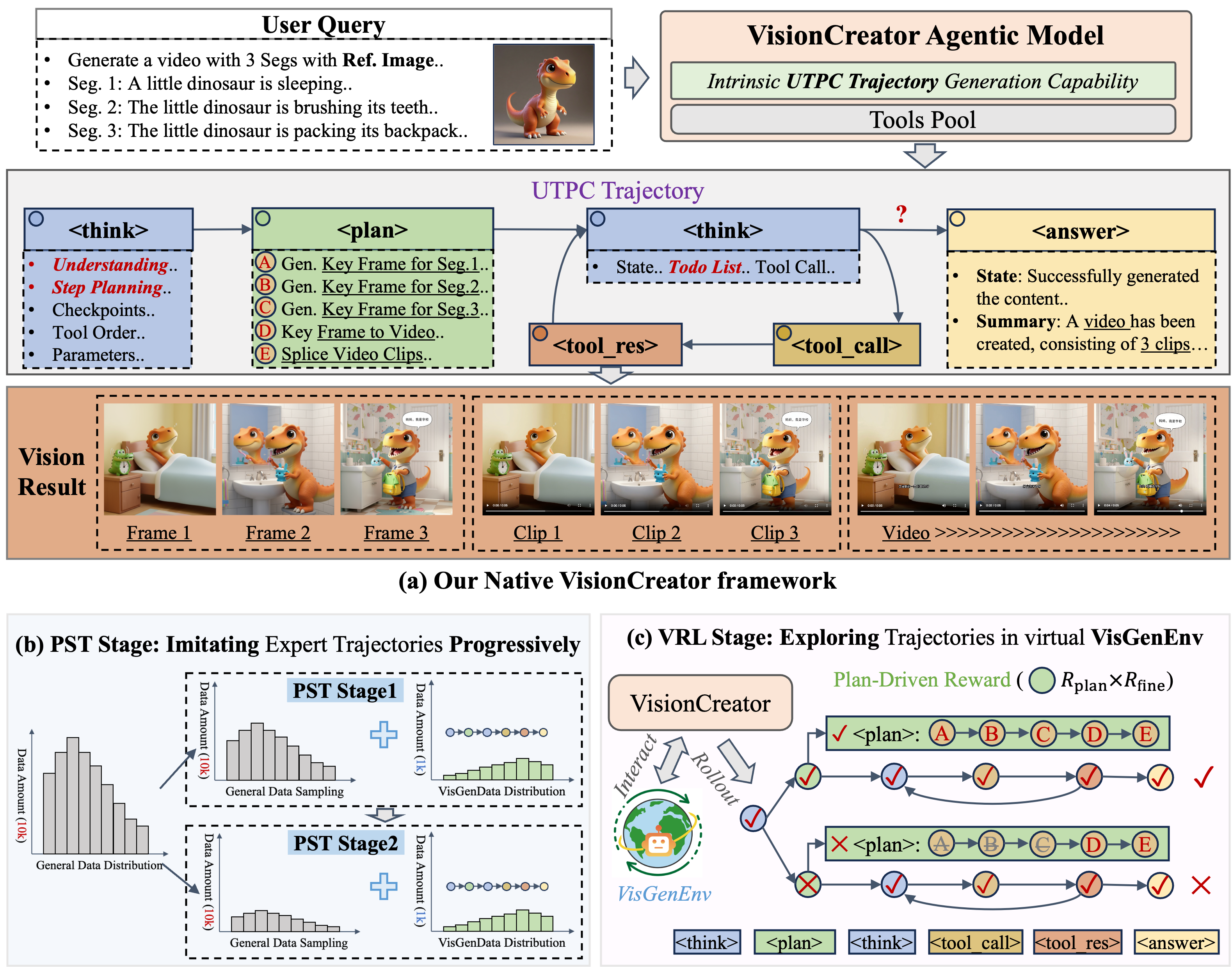}
\vspace{-5mm}
\caption{Our Native VisionCreator framework.}
\vspace{-2mm}
\label{fig:framework}
\end{figure*}

\subsection{Virtual Reinforcement Learning}
\label{subsec:vrl}

Building upon the robust foundation established by PST, we refine the model's UTPC capabilities through Virtual Reinforcement Learning (VRL) based on the GRPO algorithm. 
To enable scalable long-horizon learning without invoking real-world tools, we first construct a high-fidelity virtual environment \textit{VisGenEnv} that simulates the behavior of visual creation tools. Within this environment, LtrReward components are designed to supervise agent trajectories and guide both planning and execution. To understand that policies learned under these rewards transfer effectively to real-world scenarios, we provide a theoretical analysis of VRL. Building upon these insights, we then introduce a plan-driven reward that integrates planning and execution signals to optimize robust long-horizon visual creation performance.

\begin{figure*}[ht]
\vspace{-2mm}
\centering
\includegraphics[width=0.8\textwidth]{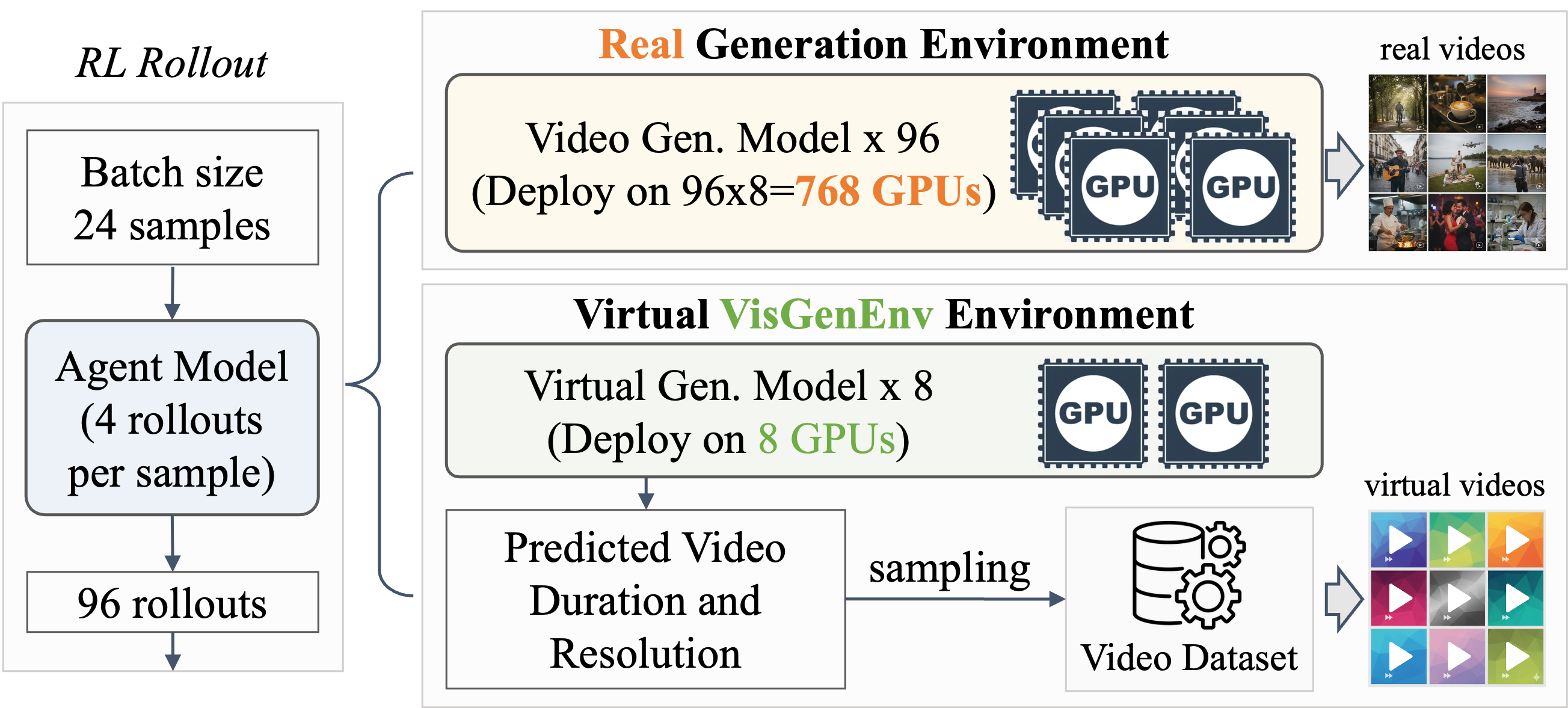}
\vspace{-2mm}
\caption{Comparison of the real environment and our virtual VisGenEnv environment, with an example of using a video generation tool.}
\vspace{-2mm}
\label{fig:virtual_env}
\end{figure*}

\subsubsection{Virtual VisGenEnv Environment}

To enable scalable long-horizon learning without invoking real-world tools, we first construct a high-fidelity virtual environment called \textit{VisGenEnv}. This environment serves as a sandbox where the agent can safely explore planning and tool-use strategies, laying the foundation for subsequent reward design and theoretical analysis.
VisGenEnv integrates a comprehensive suite of \textit{36 visual creation tools} (see Appendix for full list). The core of its design lies in a procedural simulation that accurately replicates the functional logic and behavioral patterns of real tools, including state transitions, parameter validation, and output specifications such as image resolution and video duration.
To simulate multimodal outputs, the environment returns media files randomly sampled from a database while ensuring physically correct attributes consistent with tool specifications. This high-fidelity simulation of tool behaviors enables the agent to effectively learn the causal structure of the workflow and master robust planning policies through extensive practice within the virtual setting.

Training agent models by reinforcement learning in the real environment is prohibitively expensive. As illustrated in Fig.~\ref{fig:virtual_env}, supporting a training batch size of 24 with 4 rollouts (i.e., 96 concurrent rollouts in total) quickly becomes computationally intractable. Video tools are particularly costly: each instance requires 8 GPUs and roughly 30 seconds per video, meaning 96 concurrent rollouts would require $8 \times 96 = 768$ GPUs. Deploying multiple real image and video generation tools would require several thousand GPUs, while our virtual environment VisGenEnv enables long-horizon exploration with only a few GPUs—thus saving thousands of GPU resources.

\begin{figure*}[ht]
\vspace{-2mm}
\centering
\includegraphics[width=0.88\textwidth]{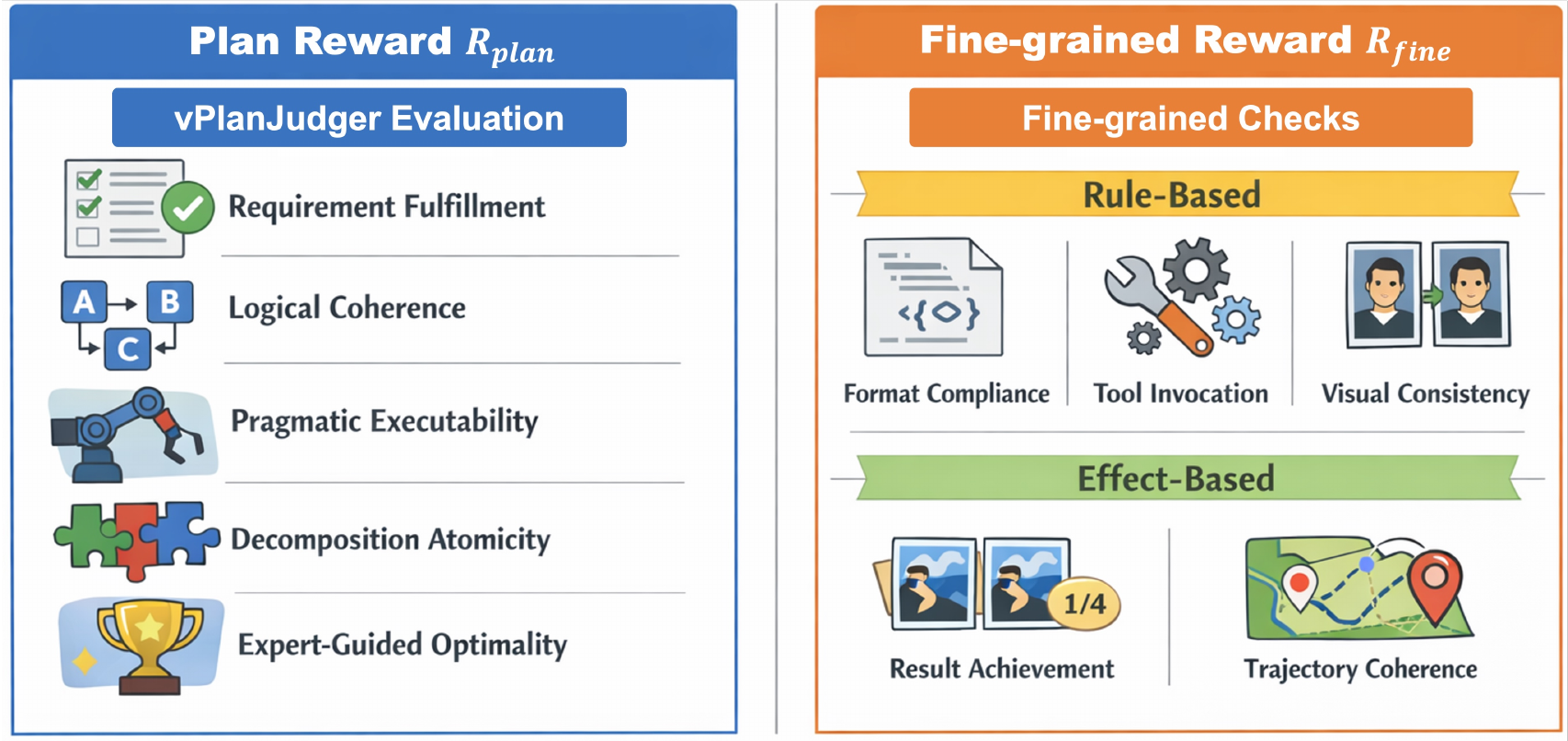}
\vspace{-2mm}
\caption{LtrReward Components.}
\vspace{-2mm}
\label{fig:reward}
\end{figure*}

\subsubsection{LtrReward Components}

With the virtual environment in place, as shown in Fig.~\ref{fig:reward}, we now define LtrReward components $R_{\text{vrt}}$ (i.e., virtual reward applicable to VisGenEnv) as reward signals that guide the agent's learning, which consist of Plan Reward \(R_{\text{plan}}\) and Fine-grained Reward \(R_{\text{fine}}\).

Plan Reward $R_{\text{plan}}$ evaluates the overall quality of the task plan using a proposed \textbf{vPlanJudger}, an expert-informed LLM evaluator that leverages a curated repository of expert reference plans to provide in-context guidance. By performing cross-referenced reasoning between the candidate plan and expert-authored strategies, the vPlanJudger computes a multidimensional alignment score focusing on five key facets: (1) \textit{Requirement Fulfillment}, a binary check on whether the output's modality and quantity align with the user request; (2) \textit{Logical Coherence}, verifying the causal validity of sub-task sequencing; (3) \textit{Pragmatic Executability}, ensuring each step is grounded within the available toolset or LLM capabilities to avoid hallucinatory actions; (4) \textit{Decomposition Atomicity}, which evaluates whether the plan is partitioned into actionable atomic tasks; and (5) \textit{Expert-Guided Optimality}, which rewards task-specific best practices such as identity consistency for multi-shot content, beat-aligned audio-visual synchronization, and the strategic minimization of complexity.

The Fine-grained Reward \(R_{\text{fine}}\) integrates both rule-based and effect-based signals to ensure structurally valid execution and successful task realization. 
Specifically: 
(1) Rule-based components include Format Compliance \(R_{\text{format}}\), which validates UTPC structural correctness via parsing of tags, ordering, content, and JSON validity; Tool Invocation \(R_{\text{tool}}\), which scores execution success with graded penalties for intermediate or final failures; and Visual Consistency \(R_{\text{cons}}\), which rewards appropriate use of reference-based generation when consistency is required. 
(2) Effect-based components include Result Achievement \(R_{\text{result}}\), which verifies output constraints such as image count and video duration within tolerance bounds, and Trajectory Coherence \(R_{\text{traj}}\), which evaluates alignment between planning intent and executed actions through an LLM-evaluator. Together, these rewards provide trajectory-level supervision that encourages correct agentic structure, reliable tool usage, and coherent visual creation outcomes.

\subsubsection{Theoretical Foundations of Virtual Reinforcement Learning}
\label{subsec:theory}

Based on the constructed virtual environment and the LtrReward components, we provide a theoretical analysis to explain the effectiveness of VRL when transferred to real-world execution.
The theoretical legitimacy of VRL rests on its ability to maintain policy efficacy despite the intrinsic discrepancies between virtual simulation and real-world execution. 
Specifically, VRL operates under a \textbf{Rollout Gap}, where the agent lacks real visual feedback to rectify its trajectory, and an \textbf{Objective Inconsistency}, caused by substituting the vision reward $R_{\text{vision}}$ (which measures perceptual quality across multiple visual dimensions) with a structural proxy $R_{\text{result}}$. To evaluate how these discrepancies affect policy transfer, we model the sim-to-real transition as a function of four synergistic variables: (i) \textbf{Tool Capability} ($C_{\text{tool}}$), quantifying the reliability of the generative engine; (ii) \textbf{PST Prior} ($\pi_{\text{pst}}$), anchoring the agent's initial reasoning within a distribution derived from real expert data; (iii) \textbf{Plan Sufficiency} ($\Phi_{\text{plan}}$), measuring the causal link between logical correctness and visual quality; and (iv) \textbf{Result Reward} ($R_{\text{result}}$), ensuring the structural completion of tasks.

The following theorems establish the mathematical foundation of VRL: Theorem~\ref{thm:error_bound} provides an error bound guarantee, proving that the sim-to-real gap remains controllable under the joint constraint of these variables; Theorem~\ref{thm:improvement} characterizes the real-world performance 
gain as a competition between \textit{Causal Improvement} and \textit{Transfer Loss}, showing that VRL yields non-negative improvement 
whenever the causal reward gain dominates the bounded sim-to-real error.

\begin{theorem}[Virtual-to-Real Error Bound]
\label{thm:error_bound}
Let $J_{\text{real}}(\pi)$ and $J_{\text{vrt}}(\pi)$ be the expected returns of policy $\pi$ in real and virtual environments. And $\delta, \alpha, \beta$ are environment-specific scaling factors. The transfer error $\mathcal{E}(\pi) = |J_{\text{real}}(\pi) - J_{\text{vrt}}(\pi)|$ is bounded by:
\begin{equation}
\small
\mathcal{E}(\pi) \le \underbrace{\delta(1 - C_{\text{tool}})}_{\text{Dynamics Gap}} + \underbrace{\alpha \cdot D_{\text{KL}}(\pi_{\text{vrt}} | \pi_{\text{pst}})}_{\text{Action Bias Bound}} + \underbrace{\beta(1 - \Phi_{\text{plan}} \cdot R_{\text{result}})}_{\text{Goal Alignment Error}}
\end{equation}
\end{theorem}
Theorem \ref{thm:error_bound} quantifies how the sim-to-real divergence is suppressed: (i) \textit{Dynamics Gap} is minimized by $C_{\text{tool}}$, ensuring virtual procedural logic mirrors real API behavior; (ii) \textit{Action Bias Bound} is constrained by the PST prior, which prevents policy drift in the absence of real visual feedback by maintaining consistency with expert decision-making; (iii) \textit{Goal Alignment Error} is mitigated by the coupling of $\Phi_{\text{plan}}$ and $R_{\text{result}}$, ensuring the virtual completion objective serves as a reliable proxy for real-world success.

\begin{theorem}[Real-World Improvement of VRL]
\label{thm:improvement}
Under the error bound $\mathcal{E}$, the real-world performance gain depends on the dominance of Causal Improvement over Transfer Loss:
\begin{equation}
J_{\text{real}}(\pi_{\text{VRL}}) - J_{\text{real}}(\pi_{\text{pst}}) \ge \underbrace{\Gamma \cdot \mathbb{E}_{\pi} [\Delta R_{\text{vrt}}]}_{\text{Causal Improvement}} - \underbrace{\mathcal{E}(\pi)}_{\text{Transfer Loss}}
\end{equation}
where $\Gamma = C_{\text{tool}} \cdot \Phi_{\text{plan}} \cdot \kappa(\pi_{\text{pst}})$ is the effectiveness coefficient, and $\kappa(\pi_{\text{pst}})$ denotes the anchoring strength of the PST prior in constraining policy exploration. Virtual reward $R_{\text{vrt}}$ consisting of $R_{\text{plan}}$ and $R_{\text{fine}}$, and $\mathbb{E}_{\pi} [\Delta R_{\text{vrt}}]$ denotes the expected increment of virtual reward, representing the agent's logic optimization in planning and execution.
\end{theorem}

The practical transferability of VRL is validated by the convergence behavior in our experiments, where the agent achieves an average virtual reward exceeding \textbf{95\%}. This saturation of total virtual reward $R_{\text{vrt}}$ indicates that the \textbf{Causal Improvement} term is maximized, providing a substantial logical buffer to offset transfer discrepancies. By substituting these empirical results into Theorem \ref{thm:error_bound}, we observe that the \textit{Action Bias Bound} is strictly suppressed by the PST prior, while the \textit{Goal Alignment Error} is mitigated by the coupling of $\Phi_{\text{plan}}$ and $R_{\text{result}}$, remaining stable as the agent masters structural completion. Consequently, the \textbf{Transfer Loss} $\mathcal{E}(\pi)$ is primarily governed by the \textit{Dynamics Gap} $\delta(1 - C_{\text{tool}})$. 
This reveals a critical insight: VRL efficacy is fundamentally a function of generative tool quality. As $C_{\text{tool}}$ increases—meaning the underlying visual creation tools become more reliable and follow procedural logic more closely—the transfer loss diminishes, allowing the massive logical gains from virtual training to translate effectively into superior real-world visual quality. 
Therefore, we derive the following corollary:
\begin{corollary}[Fidelity-Anchored Transfer]
\label{cor:fidelity}
Provided the virtual reward $R_{\text{vrt}}$ reaches a near-optimal level, the real-world gain of VRL is  monotonically non-decreasing with respect to $C_{\text{tool}}$.
\end{corollary}

\subsubsection{Plan-Driven Reward Design}
\label{subsec:plan_driven_reward}
Theorems~\ref{thm:error_bound} and~\ref{thm:improvement} indicate that real-world improvement critically depends on planning quality. Motivated by this insight, we adopt a plan-driven reward that enforces causal dependency between planning and execution:
\begin{align}
R_{\text{vrt}} = R_{\text{plan}} \times R_{\text{fine}} [R_{\text{tool}} + R_{\text{format}} + R_{\text{result}} + R_{\text{traj}} + R_{\text{cons}}].
\end{align}
Here, \(R_{\text{plan}}\) measures plan correctness, while \(R_{\text{fine}}\) captures execution-level structural validity. The multiplicative coupling ensures that execution alone cannot achieve high reward without a valid plan, and maximal reward is obtained only when a correct plan is faithfully executed. This mechanism directly aligns with Theorem~\ref{thm:improvement}, promoting robust long-horizon planning and tool-use strategies within virtual training.

\section{Experiment}
\label{sec:exp}

\subsection{VisGenBench}
\label{sec:benchmarks}
Existing video generation benchmark VBench-2.0~\cite{zheng2025vbench} has made significant contributions to evaluating the quality of individual-generated videos.
But it lacks the capability to evaluate \textit{multi-step visual creation trajectories} that involve complex tool invocation and long-horizon planning.
While ComfyBench~\cite{xue2025comfybench} attempts to assess multi-step trajectories, it is specifically designed for ComfyUI~\cite{comfyanonymous2023comfyui} and evaluates agent performance based solely on ComfyUI execution success, making it unsuitable for general API-based tool invocation scenarios.
To address this critical gap, we introduce \textbf{VisGenBench}, a comprehensive benchmark designed for evaluating \textit{visual generation agentic models} that operate through multi-step tool invocation to accomplish complex image and video creation tasks.

\begin{table}[htbp]
\centering
\renewcommand{\tabcolsep}{0.8pt}
\renewcommand{\arraystretch}{1.1}
\small
\caption{Test dataset composition of VisGenBench, with 400 image tasks and 800 video tasks.}
\label{tab:visgenbench_stats}
\vspace{-2mm}
\begin{tabular}{lccccccccc}
\toprule
\multirow{2}{*}{\textbf{Type}} & \textbf{Content} & \textbf{Content} & \textbf{Object} & \textbf{Scene} & \textbf{Style} & \multirow{2}{*}{\textbf{Variety}} & \textbf{Visual} & \textbf{Video} & \textbf{Video} \\
& \textbf{Creative} & \textbf{Match} & \textbf{Consistency} & \textbf{Consistency} & \textbf{Consistency} & & \textbf{Amount} & \textbf{Duration} & \textbf{Storyboard} \\
\midrule
\textbf{Image Tasks} & 50 & 50 & 50 & 50 & 50 & 50 & 100 & -- & -- \\
\textbf{Video Tasks} & 50 & 50 & 50 & 50 & 50 & 50 & 100 & 200 & 200 \\
\hline
\textbf{Total} & 100 & 100 & 100 & 100 & 100 & 100 & 200 & 200 & 200 \\
\bottomrule
\end{tabular}
\end{table}

\subsubsection{Test Dataset Composition}
 
As shown in Tab.~\ref{tab:visgenbench_stats}, the VisGenBench consists of a total of 1.2k test samples, including 400 image-generation tasks and 800 video-generation tasks.
Each task is designed to reflect multi-step creation trajectories, requiring to generation of many images and videos.
The benchmark spans 10 evaluation dimensions and covers 35+ real-world application scenarios, encompassing domains such as advertising, storytelling, entertainment, animation, etc.

\subsubsection{Evaluation Framework}
The VisGenBench evaluation framework integrates both objective and subjective assessments to measure an agent’s ability to perform multi-step visual generation tasks.

\noindent\textbf{Objective Evaluation}
Objective evaluation focuses on quantifiable and automatically measurable aspects of the generated content. Specifically, it consists of two components: (1) \textit{Success Rate:} Measures whether the model successfully returns valid images/videos when requested by user. A generation containing the correct modality is counted as \textit{Success}. (2) \textit{Basic Visual Attributes:} Quantitative evaluation of the generated results, including visual quantity, video storyboard count, and video duration. These attributes are automatically assessed using standardized tools.

\noindent\textbf{Subjective Evaluation}
Subjective aspects such as visual consistency, diversity, storytelling quality, and audio perception cannot be fully captured through traditional metrics. We therefore introduce a VLM-Grader with pre-defined fine-grained scoring rubrics, implemented using the Gemini2.5-Pro model. For each subjective evaluation dimension, we define a tailored \textit{meta evaluation list}—a structured rubric containing detailed scoring items (e.g., character consistency, style coherence, narrative flow, audio synchronization, etc). Gemini2.5-Pro provides a meta-evaluation score for each meta-item, and the aggregated score forms the overall result for that dimension. To align automated scoring with human judgment, we calibrate Gemini2.5-Pro's meta-evaluation intensity on VisGenBench. This ensures that both mean scores and relative rankings evaluated by Gemini2.5-Pro remain consistent with expert human assessments, achieving a human-aligned evaluation process.

\begin{table*}[htbp]
\centering
\renewcommand{\tabcolsep}{0.8pt}
\renewcommand{\arraystretch}{1.1}
\small
\caption{Comparisons on VisGenBench by VLM Evaluation. S-Rate: Success Rate, O-Score: Overall Score. The \colorbox{lightpink!90}{best} and \colorbox{lightblue!40}{second-best} results are highlighted.}
\label{tab:visgenbench}
\vspace{-3mm}
\begin{tabular}{lccccccccccc}
\toprule
\textbf{Method} & \textbf{Creative} & \textbf{Match} & \textbf{Object} & \textbf{Scene} & \textbf{Style} & \textbf{Variety} & \textbf{Amount} & \textbf{Duration} & \textbf{Storyboard} & \textbf{S-Rate} & \textbf{O-Score} \\
\midrule
GPT-5 & {0.683} & 0.641 & 0.593 & 0.579 & \textbf{0.638} & {0.232} & \textbf{0.620} & 0.263 & {0.660} & 0.863 & 0.577 \\
Gemini2.5-Pro & \textbf{0.777} & \textbf{0.802} & {0.625} & {0.602} & 0.573 & \textbf{0.345} & {0.540} & {0.376} & \textbf{0.700} & \colorbox{lightpink!90}{0.933} & \colorbox{lightpink!90}{0.627} \\
\hline
Qwen3-VL-8B-Tk & 0.104 & 0.078 & 0.100 & 0.065 & 0.109 & 0.014 & 0.160 & 0.034 & 0.040 & 0.142 & 0.085 \\
\textbf{VisionCreator-8B} & 0.651 & {0.661} & \textbf{0.645} & \textbf{0.638} & {0.595} & 0.211 & 0.480 & \textbf{0.429} & 0.580 & \colorbox{lightblue!40}{0.925} & \colorbox{lightblue!40}{0.581} \\
\bottomrule
\end{tabular}
\vspace{-3mm}
\end{table*}

\begin{table}[htbp]
\centering
\renewcommand{\tabcolsep}{3.5pt}
\renewcommand{\arraystretch}{1.1}
\small
\caption{Comparisons on VisGenBench by Human Evaluation. All models use the new version system prompt, which differs from Tab.\ref{tab:visgenbench}. Overall Score = (Success Rate of Image $\times$ Human Evaluation of Image $+$ Success Rate of Video $\times$ Human Evaluation of Video)$/$2. The performance comparisons of all detailed human evaluation dimensions are shown in Fig.~\ref{fig:radar}.}
\vspace{-2mm}
\begin{tabular}{lccccc}
\toprule
\multirow{2}{*}{\textbf{Model}} & \multicolumn{2}{c}{\textbf{Success Rate}} & \multicolumn{2}{c}{\textbf{Human Evaluation}} & \multirow{2}{*}{\textbf{Overall Score}} \\ \cline{2-5}
& \textbf{Image} & \textbf{Video} & \textbf{Image} & \textbf{Video} & \\ 
\hline
GPT-5 & 95.95\% & 93.00\% & 3.52 & 3.25 & 3.19 \\ 
Gemini2.5-Pro & 91.00\% & 84.00\% & 3.53 & 3.35 & 3.01 \\ 
\hline
Qwen3-VL-32B-Thinking & 97.00\% & 93.00\% & 3.47 & 3.23 & 3.18 \\ 
Qwen3-VL-32B-RL & 91.00\% & 87.00\% & 3.51 & 3.40 & 3.07 \\
Qwen3-VL-32B-SFT & 96.00\% & 94.00\% & 3.53 & 3.37 & 3.27 \\
VisionCreator-32B & \colorbox{lightpink!90}{99.00\%} & \colorbox{lightpink!90}{96.00\%} & \colorbox{lightpink!90}{3.53} & \colorbox{lightpink!90}{3.49} & \colorbox{lightpink!90}{3.42} \\ 
\bottomrule
\end{tabular}
\label{tab:human_eval}
\end{table}

\subsection{Results on VisGenBench by VLM Evaluation}
\label{subsec:visgenbench}
As shown in Tab.~\ref{tab:visgenbench}, our VisionCreator-8B demonstrates remarkable performance that is highly competitive with much larger commercial models (GPT-5 and Gemini2.5-Pro), while significantly outperforming its base model Qwen3-VL-8B-Thinking. The key findings highlight several advantages of our approach:
(1) \textit{Superior Success Rate and Reliability:} VisionCreator-8B achieves an impressive success rate of 0.925, surpassing GPT-5 (0.863) and approaching Gemini2.5-Pro (0.933). This demonstrates the effectiveness of our UTPC framework in ensuring task completion reliability, a crucial requirement for practical visual creation applications.
(2) \textit{Exceptional Consistency Performance:} VisionCreator-8B achieves the highest scores in object consistency (0.645) and scene consistency (0.638) among all compared models, including the much larger Gemini2.5-Pro and GPT-5. This validates our model's strong capability in maintaining visual coherence throughout multi-step creation processes, a core benefit of the native agentic architecture.
(3) The results validate our core hypothesis: a specialized native visual creation agent, even with significantly fewer parameters, can achieve performance competitive with general-purpose commercial giants through targeted architectural design and training methodology. VisionCreator's particular strengths in success rate and consistency metrics underscore the practical advantages of our UTPC framework for real-world visual content creation applications.

\begin{table*}[t]
\centering
\renewcommand{\tabcolsep}{1.8pt}
\renewcommand{\arraystretch}{1.1}
\small
\caption{Ablation study with VisionCreator-8B on VisGenBench-104 comparing different training strategies. VisGenBench-104 is a sampled subset of VisGenBench. Model configurations: 
\textbf{RL1}: PST + Result+Format reward); 
\textbf{RL2}: PST + Plan×(Result+Format) reward; 
\textbf{RL3}: Qwen3-VL + Plan×(Result+Format) reward; 
\textbf{RL4}: PST + Plan×Fine reward; 
\textbf{v1}: 3×VisGenData-4k; 
\textbf{v2}: 3×VisGenData-4k + General-1M; 
\textbf{v3}: 20×VisGenData-4k + General-1M; 
\textbf{v4}: PST + 3×VisGenData-4k + General-1\%; 
\textbf{v5}: PST + 3×VisGenData-4k + General-5\%; 
\textbf{v6}: PST + 3×VisGenData-4k + General-10\%; 
\textbf{v7}: PST + 3×VisGenData-4k + General-20\%.}
\label{tab:ablation_study}
\vspace{-3mm}
\begin{tabular}{lccccccccccc}
\toprule
\textbf{Method} & \textbf{Creative} & \textbf{Match} & \textbf{Object} & \textbf{Scene} & \textbf{Style} & \textbf{Variety} & \textbf{Amount} & \textbf{Duration} & \textbf{Storyboard} & \textbf{S-Rate} & \textbf{O-Score} \\
\midrule
\textbf{RL1} & 0.534 & \textbf{0.817} & \textbf{0.694} & 0.547 & \textbf{0.579} & 0.249 & \textbf{1.000} & 0.397 & 0.625 & 0.904 & 0.634 \\
\textbf{RL2} & 0.579 & 0.808 & 0.677 & 0.479 & 0.558 & \textbf{0.265} & 0.800 & 0.478 & \textbf{0.875} & \colorbox{lightpink!90}{0.942} & 0.644 \\
\textbf{RL3} & \textbf{0.671} & 0.674 & 0.621 & 0.622 & 0.555 & 0.217 & 0.800 & 0.513 & 0.750 & 0.885 & 0.631 \\
\textbf{RL4} & 0.573 & 0.794 & 0.672 & \textbf{0.696} & 0.569 & 0.150 & \textbf{1.000} & \textbf{0.534} & 0.625 & 0.925 & \colorbox{lightpink!90}{0.654} \\
\hline
\textbf{v1} & 0.000 & 0.050 & 0.000 & 0.000 & 0.000 & 0.000 & 0.000 & 0.000 & 0.000 & 0.019 & 0.007 \\
\textbf{v2} & 0.230 & 0.334 & 0.382 & 0.339 & 0.396 & 0.163 & 0.600 & 0.134 & 0.500 & 0.490 & 0.357 \\
\textbf{v3} & 0.262 & 0.422 & 0.300 & 0.260 & 0.473 & 0.068 & 0.600 & 0.100 & 0.250 & 0.481 & 0.322 \\
\textbf{v4} & 0.283 & 0.468 & 0.399 & 0.295 & 0.318 & 0.000 & 0.600 & 0.183 & 0.000 & 0.442 & 0.299 \\
\textbf{v5} & 0.266 & 0.361 & 0.366 & 0.246 & 0.201 & 0.084 & 0.200 & 0.029 & 0.125 & 0.490 & 0.237 \\
\textbf{v6} & 0.239 & 0.310 & 0.326 & 0.194 & 0.273 & 0.149 & 0.600 & 0.098 & 0.125 & 0.413 & 0.273 \\
\textbf{v7} & 0.420 & 0.701 & 0.430 & 0.447 & 0.430 & 0.028 & 0.600 & 0.344 & 0.375 & \colorbox{lightpink!90}{0.625} & \colorbox{lightpink!90}{0.440} \\
\bottomrule
\end{tabular}
\vspace{-2mm}
\end{table*}

\subsection{Results on VisGenBench by Human Evaluation}
\label{subsec:human_eval_analysis}

In addition to automated VLM-based evaluation (Tab.~\ref{tab:visgenbench}), we conduct a thorough \textit{human evaluation} to assess the perceptual quality of multi-step visual creation tasks, including images and videos (Tab.~\ref{tab:human_eval}), which shows that:
(1) \noindent\textbf{Overall Findings:} VisionCreator-32B achieves the highest Overall Score of 3.42, surpassing both GPT-5 (3.19) and Gemini2.5-Pro (3.01). This indicates that our UTPC framework not only ensures task success in an automated setting but also delivers outputs that are qualitatively preferred by human evaluators.
(2) \noindent\textbf{Image vs. Video Performance:} VisionCreator-32B excels across both modalities, with 99\% image success and 96\% video success, accompanied by strong human evaluation scores (3.53 for images, 3.49 for videos). This balanced performance highlights the model's capability to maintain coherent multi-step planning and execution for both static and dynamic content.
(3) \noindent\textbf{Implications:} The human evaluation corroborates trends observed in VLM-based metrics, validating that the model's planning-driven reward design and VRL training not only improve automated success metrics but also enhance perceptual quality, consistency, and user satisfaction in real-world multi-step visual creation tasks.

\subsection{Ablation Studies}
\label{subsec:ablation_studies}
We conduct ablation studies on sampled VisGenBench-104, where key findings from Tab.~\ref{tab:ablation_study} include:
(1) \textbf{Effectiveness of PST}.
Our PST with \textit{v7 (PST + 3×VisGenData-4k + General-20\%)} achieves significant improvement over SFT with \textit{v2 (3×VisGenData-4k + General-1M)} (0.440 vs. 0.357). Performance improves with increasing general data ratio (\textit{v4}→\textit{v5}→\textit{v6}→\textit{v7}), confirming balanced specialization prevents overfitting while maintaining generalization.
(2) \textbf{Data Configuration Strategies}.
Simply increasing specialized data scale does not guarantee improvement. \textit{v3 (20×VisGenData-4k + General-1M)} underperforms \textit{v2 (3×VisGenData-4k + General-1M)} (0.322 vs. 0.357), indicating excessive data repetition causes overfitting. Our PST strategy achieves better performance through optimized data ratios.
(3) \textbf{Virtual Reinforcement Learning}.
All VRL models substantially outperform SFT variants. \textit{RL4 (PST + Plan×Fine reward)} improves Overall Score by 49\% over the best PST model \textit{v7} (0.654 vs. 0.440), demonstrating VRL's effectiveness.
(4) \textbf{Reward Function Designs}. Building upon RL1, \textit{RL2 (PST + Plan×(Result+Format) reward)} which incorporates additional plan reward, demonstrates improved performance with a higher Success Rate (0.942 vs. 0.904) and Overall Score (0.644 vs. 0.634). \textit{RL4} achieves the best Overall Score (0.654) and demonstrates strong comprehensive performance across multiple dimensions, proving fine-grained rewards enhance model capability.
(5) \textbf{Importance of Pre-training Foundation}
\textit{RL2 (PST + Plan×(Result+Format) reward)} outperforms \textit{RL3 (Qwen3-VL + Plan×(Result+Format) reward)} (0.644 vs. 0.631) despite identical rewards, with \textit{RL2} achieving a notably higher Success Rate of 0.942 compared to 0.885 for \textit{RL3}, validating PST provides a stronger foundation for RL training.

\begin{figure*}[ht]
\vspace{-1mm}
\centering
\includegraphics[width=0.95\textwidth]{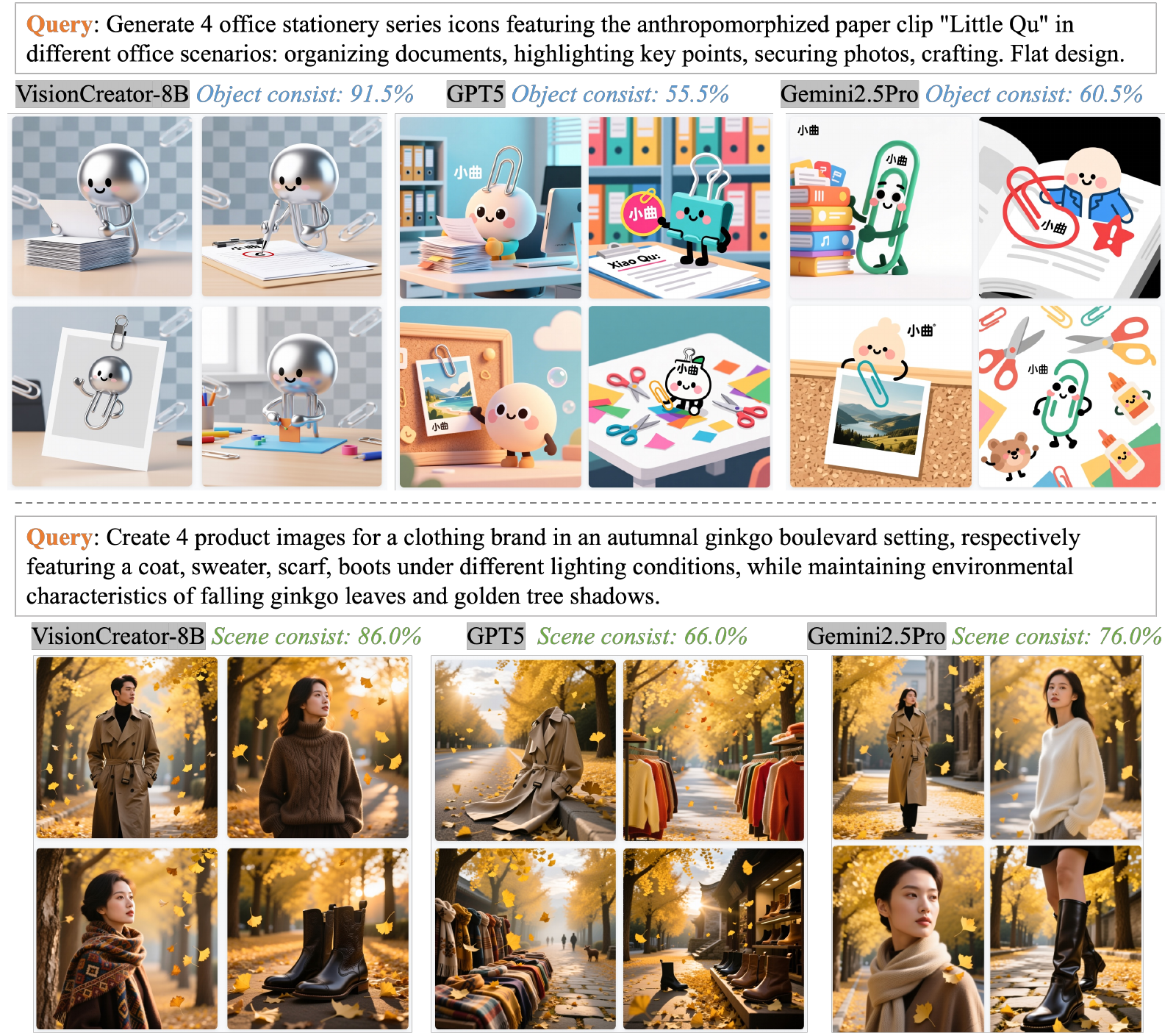}
\vspace{-2mm}
\caption{Visualization comparisons of consistency.}
\vspace{-3mm}
\label{fig:vis_compare}
\end{figure*}

\section{Conclusion}
\label{sec:conclusion}
We present VisionCreator, a native visual-generation agent that unifies Understanding, Thinking, Planning, and Creation (UTPC) in an end-to-end framework. Our contributions include: (1) VisGenData-4k with UTPC structures via metacognition-based VisionAgent; (2) Progressive Specialization Training and Virtual Reinforcement Learning for stable capability acquisition; (3) VisGenBench for multi-step visual creation evaluation. Experiments show VisionCreator outperforms larger closed-source models, validating our approach. This work establishes a foundation for visual-generation agentic systems and autonomous creative content generation.




\clearpage
\setcounter{section}{0}
\renewcommand{\thesection}{\Alph{section}}
\renewcommand{\thesubsection}{\thesection.\arabic{subsection}}

\section*{Detailed Theoretical Derivations of VRL Theorems}
\label{app:theoretical_derivations}

This appendix provides detailed mathematical derivations and proofs for the two VRL theorems presented in the main text. The derivation process is divided into three stages: formal modeling and definitions, derivation of the error upper bound (Theorems~\ref{thm:error_bound}), and analysis of performance improvement (Theorems~\ref{thm:improvement}).

\subsection*{Stage 1: Formal Modeling and Definitions}

We first formalize the agent's policy, environment, and rewards to establish the foundation for subsequent derivations.

\subsubsection*{1.1 Formalization of Environment and Policy}

\textbf{Definition A.1 (MDP Tuple)}: The real-world task is modeled as a Markov Decision Process (MDP), denoted as $\mathcal{M}_{\text{real}} = (\mathcal{S}, \mathcal{A}, P_{\text{real}}, R_{\text{vision}}, \rho_0, \gamma)$.
\begin{itemize}
    \item $\mathcal{S}$: State space, containing multimodal observations $o_t$ (textual instructions, tool feedback, virtual visual states).
    \item $\mathcal{A}$: Action space, containing reasoning tokens, planning steps, and tool invocations.
    \item $P_{\text{real}}(s'|s,a)$: Dynamic transition probability of the real environment.
    \item $R_{\text{vision}}(s,a,s')$: Real reward function, measuring the perceptual quality of generated content (e.g., aesthetics, alignment).
    \item $\rho_0$: Initial state distribution.
    \item $\gamma \in (0,1)$: Discount factor.
\end{itemize}

\textbf{Definition A.2 (Virtual Environment)}: The virtual environment is $\mathcal{M}_{\text{vrt}} = (\mathcal{S}, \mathcal{A}, P_{\text{vrt}}, R_{\text{vrt}}, \rho_0, \gamma)$. Its core differences are:
\begin{itemize}
    \item $P_{\text{vrt}}(s'|s,a)$: Tool dynamics simulated by VisGenEnv, with fidelity quantified by the \textbf{tool capability} $C_{\text{tool}} \in [0,1]$.
    \item $R_{\text{vrt}}$: Virtual reward function, composed of $R_{\text{plan}}$ and $R_{\text{fine}}$ according to the plan-driven reward design. It is a structural proxy reward that substitutes for the computationally infeasible $R_{\text{vision}}$ in the virtual environment.
\end{itemize}

\textbf{Definition A.3 (Policy and Return)}: Let $\pi$ be a policy (mapping from states to actions). The expected discounted return of policy $\pi$ in environment $\mathcal{M}$ is defined as:
\[
J(\pi; \mathcal{M}) = \mathbb{E}_{\tau \sim (\pi, \mathcal{M})}\left[ \sum_{t=0}^{\infty} \gamma^t R(s_t, a_t, s_{t+1}) \right].
\]
Here, the trajectory $\tau = (s_0, a_0, s_1, a_1, \dots)$ is generated by $s_0 \sim \rho_0$, $a_t \sim \pi(\cdot|s_t)$, and $s_{t+1} \sim P(\cdot|s_t, a_t)$. For brevity, we denote $J_{\text{real}}(\pi) = J(\pi; \mathcal{M}_{\text{real}})$ and $J_{\text{vrt}}(\pi) = J(\pi; \mathcal{M}_{\text{vrt}})$.

\subsubsection*{1.2 Key Variables and Core Assumptions}

\textbf{Definition A.4 (Key Variables)}:
\begin{itemize}
    \item \textbf{Tool capability $C_{\text{tool}}$}: Measures how well the virtual environment dynamics $P_{\text{vrt}}$ approximate the real dynamics $P_{\text{real}}$. $C_{\text{tool}}=1$ indicates perfect simulation.
    \item \textbf{PST prior $\pi_{\text{pst}}$}: The initialization policy obtained through Progressive Specialization Training (PST). Its behavioral distribution on real expert data is denoted as $d_{\text{pst}}(s,a)$.
    \item \textbf{Plan sufficiency $\Phi_{\text{plan}} \in [0,1]$}: Measures the strength of the causal link between a "logically correct" plan and the final "high-quality visual output".
    \item \textbf{Result reward $R_{\text{result}} \in [0,1]$}: A subcomponent of $R_{\text{fine}}$ that evaluates whether the task is structurally completed (e.g., number of images, video duration).
\end{itemize}

\textbf{Assumption A.1 (Dynamic Difference Upper Bound)}: There exists a constant $\delta > 0$ related to environment complexity such that for all state-action pairs $(s,a)$,
\[
| P_{\text{real}}(\cdot|s,a) - P_{\text{vrt}}(\cdot|s,a) |_1 \leq \delta (1 - C_{\text{tool}}).
\]
This assumption stems from the high-fidelity simulation design of VisGenEnv: higher tool capability $C_{\text{tool}}$ leads to smaller differences between virtual and real transitions.

\textbf{Assumption A.2 (Reward Proxy Error)}: The relationship between the real visual reward $R_{\text{vision}}$ and the proxy reward is modulated by plan sufficiency $\Phi_{\text{plan}}$ and result reward $R_{\text{result}}$. There exists a constant $\beta > 0$ such that for meaningful trajectories (i.e., when planning logic is correct), the reward difference satisfies:
\[
| R_{\text{vision}}(s,a,s') - \Phi_{\text{plan}} \cdot R_{\text{result}}(s,a,s') | \leq \beta (1 - \Phi_{\text{plan}} \cdot R_{\text{result}}).
\]
This assumption reflects the design philosophy of LtrReward: when planning is sufficient ($\Phi_{\text{plan}} \approx 1$) and the task is perfectly completed structurally ($R_{\text{result}} \approx 1$), the real visual quality also tends to be high.

\textbf{Assumption A.3 (KL Constraint on Policy Deviation)}: The policy $\pi_{\text{vrt}}$ trained in the virtual environment differs from the PST prior $\pi_{\text{pst}}$ in the state-action distribution. This difference can be measured by the KL divergence $D_{\text{KL}}(\pi_{\text{vrt}} | \pi_{\text{pst}})$, and its impact on the return difference is linearly bounded. That is, there exists a constant $\alpha > 0$ such that the related performance difference is constrained by it.

\textbf{Definition A.5 (Sim-to-Real Error)}: For a given policy $\pi$, its sim-to-real error is defined as:
\[
\mathcal{E}(\pi) = | J_{\text{real}}(\pi) - J_{\text{vrt}}(\pi) |.
\]

\subsection*{Stage 2: Derivation of Theorems~\ref{thm:error_bound} (Virtual-to-Real Error Upper Bound)}

\textbf{Theorems~\ref{thm:error_bound} (Virtual-to-Real Error Upper Bound)} Restated:
Under Assumptions A.1, A.2, A.3, for any policy $\pi$ (trained in the virtual environment, denoted as $\pi_{\text{vrt}}$), its sim-to-real error $\mathcal{E}(\pi)$ satisfies:
\[
\mathcal{E}(\pi) \le \underbrace{\delta(1 - C_{\text{tool}})}_{\text{Dynamics Gap}} + \underbrace{\alpha \cdot D_{\text{KL}}(\pi_{\text{vrt}} | \pi_{\text{pst}})}_{\text{Action Bias Bound}} + \underbrace{\beta(1 - \Phi_{\text{plan}} \cdot R_{\text{result}})}_{\text{Goal Alignment Error}}.
\]

\textbf{Proof}:

We decompose the total error into three separately bounded components via the triangle inequality and constrain each using the above assumptions.

\noindent \textbf{Step 2.1: Decompose Total Error}
Consider an intermediate environment $\mathcal{M}_{\text{hybrid}} = (\mathcal{S}, \mathcal{A}, P_{\text{vrt}}, $ $R_{\text{vision}}, \rho_0, \gamma)$, which uses the virtual environment dynamics $P_{\text{vrt}}$ but retains the real reward $R_{\text{vision}}$. Denote $J_{\text{hybrid}}(\pi) = J(\pi; \mathcal{M}_{\text{hybrid}})$. Then:
\[
\begin{aligned}
\mathcal{E}(\pi) &= |J_{\text{real}}(\pi) - J_{\text{vrt}}(\pi)| \\
&\leq |J_{\text{real}}(\pi) - J_{\text{hybrid}}(\pi)| \quad &\text{(Term I: Dynamics Gap)} \\
&+ |J_{\text{hybrid}}(\pi) - J_{\text{vrt}}^{\text{ideal}}(\pi)| &\text{(Term II: Reward Gap)} \\
&+ |J_{\text{vrt}}^{\text{ideal}}(\pi) - J_{\text{vrt}}(\pi)|. &\text{(Term III: Policy Bias)}
\end{aligned}
\]
Here $J_{\text{vrt}}^{\text{ideal}}(\pi)$ represents the ideal return under dynamics $P_{\text{vrt}}$ and reward $R_{\text{vrt}}$ with the policy perfectly constrained by the PST prior (no deviation). We next upper bound each term.

\noindent \textbf{Step 2.2: Bounding Term I (Dynamics Gap)}
Term I measures the return difference due to the difference between dynamic models $P_{\text{real}}$ and $P_{\text{vrt}}$. According to \textbf{Assumption A.1} and the \textbf{Performance Difference Lemma}, for any policy $\pi$,
\[
|J_{\text{real}}(\pi) - J_{\text{hybrid}}(\pi)| \leq \frac{\gamma \cdot \delta (1 - C_{\text{tool}})}{(1-\gamma)^2} \cdot \max_{s,a} |R_{\text{vision}}(s,a)|.
\]
Let $R_{\text{max}} = \max_{s,a} |R_{\text{vision}}(s,a)|$ and define $\delta' = \frac{\gamma R_{\text{max}}}{(1-\gamma)^2} \delta$, we obtain:
\[
\text{Term I} \leq \delta' (1 - C_{\text{tool}}).
\]
In the theorem statement, constant factors are absorbed into $\delta$, so we have Term I $\leq \delta (1 - C_{\text{tool}})$.

\noindent \textbf{Step 2.3: Bounding Term II (Reward Gap)}
Term II measures the difference between using the real reward $R_{\text{vision}}$ and using the proxy reward $\Phi_{\text{plan}} \cdot R_{\text{result}}$ (as the core part of $R_{\text{vrt}}$) under the same dynamics. According to \textbf{Assumption A.2}, for each step in the trajectory, the reward difference is bounded. Applying the Performance Difference Lemma (reward difference part) again yields:
\[
|J_{\text{hybrid}}(\pi) - J_{\text{vrt}}^{\text{ideal}}(\pi)| \leq \frac{\beta (1 - \Phi_{\text{plan}} \cdot R_{\text{result}})}{1-\gamma}.
\]
Define $\beta' = \beta / (1-\gamma)$, then Term II $\leq \beta' (1 - \Phi_{\text{plan}} \cdot R_{\text{result}})$. In the theorem statement, $\beta'$ is written as $\beta$.

\noindent \textbf{Step 2.4: Bounding Term III (Policy Bias)}
Term III measures the return loss due to the deviation of the virtually trained policy $\pi_{\text{vrt}}$ from the ideal PST prior $\pi_{\text{pst}}$. According to \textbf{Assumption A.3}, there exists a constant $\alpha > 0$ such that:
\[
|J_{\text{vrt}}^{\text{ideal}}(\pi) - J_{\text{vrt}}(\pi)| \leq \alpha \cdot D_{\text{KL}}(\pi_{\text{vrt}} | \pi_{\text{pst}}).
\]
This assumption stems from the "anchoring" effect of the PST prior on the policy exploration space, preventing catastrophic policy drift in the absence of real visual feedback.

\noindent \textbf{Step 2.5: Combining Error Upper Bounds}
Summing the upper bounds of Term I, II, and III, we obtain:
\[
\mathcal{E}(\pi) \leq \delta' (1 - C_{\text{tool}}) + \beta' (1 - \Phi_{\text{plan}} \cdot R_{\text{result}}) + \alpha \cdot D_{\text{KL}}(\pi_{\text{vrt}} | \pi_{\text{pst}}).
\]
Relabeling constants $\delta' \to \delta$, $\beta' \to \beta$ yields the form in Theorems~\ref{thm:error_bound}. $\blacksquare$

\subsection*{Stage 3: Derivation of Theorems~\ref{thm:improvement} (Real-World Performance Improvement Lower Bound)}

\textbf{Theorems~\ref{thm:improvement} (Real-World Improvement of VRL)} Restated:
Let $\pi_{\text{pst}}$ be the initial policy after PST training, and $\pi_{\text{VRL}}$ be the policy optimized through Virtual Reinforcement Learning (VRL). Define the virtual optimization gain as $\Delta_{\text{vrt}} = J_{\text{vrt}}(\pi_{\text{VRL}}) - J_{\text{vrt}}(\pi_{\text{pst}}) = \mathbb{E}_{\pi}[\Delta(R_{\text{plan}}+R_{\text{fine}})]$. Then, under the error bound of Theorems~\ref{thm:error_bound}, the real-world performance improvement satisfies:
\[
J_{\text{real}}(\pi_{\text{VRL}}) - J_{\text{real}}(\pi_{\text{pst}}) \ge \underbrace{\Gamma \cdot \Delta_{\text{vrt}}}_{\text{Causal Improvement}} - \underbrace{\mathcal{E}(\pi_{\text{VRL}})}_{\text{Transfer Loss}},
\]
where $\Gamma = C_{\text{tool}} \cdot \Phi_{\text{plan}} \cdot \kappa(\pi_{\text{pst}})$ is the effectiveness coefficient, and $\kappa(\pi_{\text{pst}}) \in (0,1]$ denotes the \textbf{Anchoring Strength} of the PST prior in constraining policy exploration.

\textbf{Proof}:

\noindent \textbf{Step 3.1: Establish Inequality Based on Error Bound}
From Theorems~\ref{thm:error_bound}, for any policy $\pi$, we have $J_{\text{real}}(\pi) \ge J_{\text{vrt}}(\pi) - \mathcal{E}(\pi)$. Applying this inequality to $\pi_{\text{VRL}}$ and $\pi_{\text{pst}}$ respectively:
\[
\begin{aligned}
J_{\text{real}}(\pi_{\text{VRL}}) &\ge J_{\text{vrt}}(\pi_{\text{VRL}}) - \mathcal{E}(\pi_{\text{VRL}}), \\
J_{\text{real}}(\pi_{\text{pst}}) &\ge J_{\text{vrt}}(\pi_{\text{pst}}) - \mathcal{E}(\pi_{\text{pst}}).
\end{aligned}
\]
Subtracting the second inequality from the first yields:
\[
J_{\text{real}}(\pi_{\text{VRL}}) - J_{\text{real}}(\pi_{\text{pst}}) \ge \left[ J_{\text{vrt}}(\pi_{\text{VRL}}) - J_{\text{vrt}}(\pi_{\text{pst}}) \right] - \left[ \mathcal{E}(\pi_{\text{VRL}}) - \mathcal{E}(\pi_{\text{pst}}) \right].
\]
Since $\pi_{\text{pst}}$ itself is trained on real expert data, its sim-to-real error $\mathcal{E}(\pi_{\text{pst}})$ is expected to be small (aligned during PST). Therefore, the lower bound of performance improvement is mainly affected by the error $\mathcal{E}(\pi_{\text{VRL}})$ of $\pi_{\text{VRL}}$. Conservatively setting the transfer loss term as $\mathcal{E}(\pi_{\text{VRL}})$ gives:
\[
J_{\text{real}}(\pi_{\text{VRL}}) - J_{\text{real}}(\pi_{\text{pst}}) \ge \Delta_{\text{vrt}} - \mathcal{E}(\pi_{\text{VRL}}). \quad \text{(1)}
\]

\noindent \textbf{Step 3.2: Relating Virtual Gain to Real Gain (Causal Improvement)}
The $\Delta_{\text{vrt}}$ in inequality (1) is the gain in virtual reward. We need to relate it to real performance improvement. This relies on a core idea: optimizing "planning and execution logic" in the virtual environment, as long as the simulation is sufficiently credible, causally leads to improved real-world visual quality.
Define the \textbf{effectiveness coefficient} $\Gamma$, which quantifies the expected increment in real reward per unit increment in virtual reward. We model it as the product of three key factors:
\begin{itemize}
    \item $C_{\text{tool}}$: Tool capability determines the probability of logical execution being reproduced in reality.
    \item $\Phi_{\text{plan}}$: Plan sufficiency determines the strength of association between correct logic and high-quality output.
    \item $\kappa(\pi_{\text{pst}})$: Anchoring strength of the PST prior, indicating the degree to which the policy remains in a "reasonable" distribution region during VRL optimization, with $\kappa \in (0,1]$. Strong anchoring ($\kappa \approx 1$) ensures the optimization direction remains effective in the real world.
\end{itemize}
Therefore, we assume a monotonic relationship:
\[
J_{\text{real}}(\pi_{\text{VRL}}) - J_{\text{real}}(\pi_{\text{pst}}) \ge \Gamma \cdot \Delta_{\text{vrt}} - \mathcal{E}(\pi_{\text{VRL}}), \quad \text{where } \Gamma = C_{\text{tool}} \cdot \Phi_{\text{plan}} \cdot \kappa(\pi_{\text{pst}}). \quad \text{(2)}
\]
When $\Gamma > 0$, the logical improvement brought by virtual optimization can be partially translated into real-world improvement.

\noindent \textbf{Step 3.3: Derive the Final Lower Bound}
Substituting $\Delta_{\text{vrt}} = \mathbb{E}_{\pi}[\Delta R_{\text{vrt}}]$ into inequality (2) yields the lower bound stated in Theorems~\ref{thm:improvement}:
\[
J_{\text{real}}(\pi_{\text{VRL}}) - J_{\text{real}}(\pi_{\text{pst}}) \ge \Gamma \cdot \mathbb{E}_{\pi} [\Delta R_{\text{vrt}}] - \mathcal{E}(\pi_{\text{VRL}}).
\]

\noindent \textbf{Step 3.4: Condition for Non-Negative Improvement}
From the inequality in Theorems~\ref{thm:improvement}, the \textbf{sufficient condition} for non-negative improvement in real-world performance (i.e., $J_{\text{real}}(\pi_{\text{VRL}}) \ge J_{\text{real}}(\pi_{\text{pst}})$) is directly obtained as:
\[
\Gamma \cdot \mathbb{E}_{\pi} [\Delta R_{\text{vrt}}] \ge \mathcal{E}(\pi_{\text{VRL}}).
\]
This means that the \textbf{Causal Improvement} brought by virtual training must be sufficient to cover the \textbf{Transfer Loss} arising from simulation imperfections. This does not require $C_{\text{tool}} = 1$ or $\Phi_{\text{plan}} = 1$; as long as their product combined with the anchoring strength is large enough to make $\Gamma$ sufficiently large, and VRL can effectively increase $\Delta R_{\text{vrt}}$ (as shown in experiments where virtual reward exceeds 95\%), positive transfer is guaranteed. $\blacksquare$


\subsection*{Summary}
Through formal modeling, this derivation decomposes the challenge of sim-to-real transfer into differences at the dynamic, reward, and policy levels, and quantifies their upper bounds using key variables such as tool capability, plan sufficiency, and PST prior. Theorems~\ref{thm:error_bound} shows that systematic error can be controlled by improving tool fidelity, strengthening PST anchoring, and optimizing plan-result alignment. Theorems~\ref{thm:improvement} further proves that as long as virtual training can effectively enhance the agent's logical capabilities (Causal Improvement) and this improvement outweighs the bounded systematic error (Transfer Loss), performance improvement in the real world is guaranteed. This provides a solid theoretical foundation for the application of virtual reinforcement learning in high-dimensional, long-horizon tasks such as visual creation.

\clearpage

\begin{table}[htbp]
\centering
\renewcommand{\tabcolsep}{0.8pt}
\renewcommand{\arraystretch}{1.1}
\small
\caption{Human Evaluation of Detailed Dimensions on VisGenBench-Image (Score = Success Rate $\times$ Human Evaluation Score)}
\begin{tabular}{lccccccccc}
\toprule
\multirow{2}{*}{\textbf{Model}} & \textbf{Semantic} & \textbf{Style} & \textbf{Emotion} & \textbf{Subject} & \textbf{Design} & \textbf{Visual} & \textbf{Text} & \multirow{2}{*}{\textbf{Creativity}} & \multirow{2}{*}{\textbf{Overall}} \\
& \textbf{Matching} & \textbf{Matching} & \textbf{Matching} & \textbf{Consistency} & \textbf{Integrity} & \textbf{Integrity} & \textbf{Quality} & & \\
\midrule
GPT-5 & 3.4883 & 3.6214 & 2.9656 & 3.6024 & 3.4408 & 3.4218 & 2.7565 & 3.4408 & 3.3458 \\
Gemini2.5-Pro & 3.3943 & 3.5399 & 2.8119 & 3.4034 & 3.2214 & 3.2669 & 2.7300 & 3.3215 & \textbf{3.2123} \\
Qwen3-VL-32B-Tk & 3.4435 & 3.7248 & 2.9876 & 3.5890 & 3.4047 & 3.4823 & 2.8130 & 3.4726 & 3.3659 \\
Qwen3-VL-32B-SFT & 3.3504 & 3.8016 & 2.8896 & 3.7632 & 3.4368 & 3.5040 & 2.8224 & 3.5232 & \textbf{3.3888} \\
VisionCreator-32B & 3.6432 & 3.8412 & 3.1581 & 3.7620 & 3.4452 & 3.6531 & 2.8809 & 3.5739 & \colorbox{lightpink!90}{\textbf{3.4947}} \\
\bottomrule
\end{tabular}
\label{tab:visgenbench_human_eval}
\end{table}

\begin{table}[htbp]
\centering
\renewcommand{\tabcolsep}{3.8pt}
\renewcommand{\arraystretch}{1.1}
\small
\caption{Human Evaluation of Detailed Dimensions on VisGenBench-Video (Score = Success Rate $\times$ Human Evaluation Score)}
\begin{adjustbox}{max width=\textwidth}
\begin{tabular}{lcccccc}
\toprule
\multirow{2}{*}{\textbf{Model}} & \multirow{2}{*}{\textbf{Script}} & \textbf{Story-} & \textbf{Content} & \textbf{Subject} & \textbf{Video} & \textbf{Visual} \\
& & \textbf{board} & \textbf{Consistency} & \textbf{Consistency} & \textbf{Effect} & \textbf{Motion} \\
\midrule
GPT-5 & 3.1062 & 2.9202 & 3.1434 & 3.1713 & 3.0597 & 3.0039 \\
Gemini2.5-Pro & 2.9484 & 2.6796 & 3.0156 & 2.856 & 2.8728 & 2.6628 \\
Qwen3-VL-32B-Thinking & 3.069 & 2.9016 & 3.1713 & 3.162 & 2.9574 & 2.9388 \\
Qwen3-VL-32B-SFT & 3.6002 & 2.867 & 3.5814 & 3.3652 & 3.243 & 2.9328 \\
VisionCreator-32B & 3.5616 & 3.1872 & 3.5808 & 3.4176 & 3.4752 & 3.2256 \\
\bottomrule
\end{tabular}
\end{adjustbox}
\begin{adjustbox}{max width=\textwidth}
\begin{tabular}{lccccccc}
\toprule
\textbf{Model} & \textbf{Audio-Visual} & \textbf{Music} & \textbf{Dubbing} & \textbf{Subtitle} & \textbf{Transition} & \textbf{Editing} & \textbf{Overall} \\
\midrule
GPT-5 & 3.0411 & 3.2643 & 2.8644 & 2.9788 & 2.8812 & 2.8392 & 2.814 \\
Gemini2.5-Pro & 2.8056 & 2.7888 & 2.8728 & 2.8812 & 2.8392 & 2.562 & 2.814 \\
Qwen3-VL-32B-Thinking & 2.9481 & 2.9202 & 3.1341 & 3.069 & 2.9295 & 2.8737 & 3.0039 \\
Qwen3-VL-32B-SFT & 3.1772 & 3.0644 & 3.2148 & 3.0174 & 3.0268 & 2.9328 & 3.1678 \\
VisionCreator-32B & 3.3792 & 3.2928 & 3.3888 & 3.3216 & 3.2352 & 3.1104 & \colorbox{lightpink!90}{3.3504} \\
\bottomrule
\end{tabular}
\end{adjustbox}
\label{tab:visgenbench_video_human_eval_split}
\end{table}

\begin{table}[h]
\centering
\renewcommand{\tabcolsep}{5.5pt}
\renewcommand{\arraystretch}{1.1}
\caption{General-purpose Datasets.}
\label{tab:ai_training_data}
\begin{tabular}{lll}
\toprule
\textbf{Category} & \textbf{Name} & \textbf{Quantity} \\
\midrule
\multirow{3}{*}{\centering NLP} 
& \multirow{1}{*}{DeepSeek-R1-Distill-110k \cite{deepseek_r1_distill_110k_sft}} & \multirow{1}{*}{110k} \\
& LONGCOT-Refine-500K \cite{longcot_refine_500k} & 500k \\
& \multirow{1}{*}{alpaca-gpt4-data \cite{alpaca_gpt4_en,alpaca_gpt4_zh}} & \multirow{1}{*}{100k} \\
\midrule
\multirow{1}{*}{\centering Multimodal} 
& \multirow{1}{*}{M3IT \cite{m3it_dataset}} & \multirow{1}{*}{1592k} \\
\midrule
\multirow{6}{*}{\centering Tool Calling} 
& function-calling-chatml \cite{function_calling_chatml_dataset} & 112k \\
& xlam-function-calling-60k \cite{xlam_function_calling_60k} & 60k \\
& ms-agent \cite{msagent} & 600k \\
& ToolACE \cite{toolace_dataset} & 11k \\
& ToolBench \cite{toolbench_dataset} & 123k \\
& AFM \cite{agent_foundation_models_repo} & 76k \\
\bottomrule
\end{tabular}
\end{table}

\begin{table}[htbp]
\centering
\caption{Task Distribution in VisGenData-4k Dataset}
\label{tab:all_tasks}
\begin{tabular}{p{0.8cm}p{6cm}p{0.8cm}p{6cm}}
\toprule
\multicolumn{2}{c}{\textbf{Video Generation Tasks}} & \multicolumn{2}{c}{\textbf{Image Generation Tasks}} \\
\cmidrule(r){1-2} \cmidrule(l){3-4}
No. & Task Type & No. & Task Type \\
\midrule
1. & Product marketing videos & 1. & Product images \\
2. & Public service advertisements & 2. & Detail pages \\
3. & Corporate promotion videos & 3. & Key Visual (KV) \\
4. & Brand story videos & 4. & Landing pages / H5 graphics \\
5. & Event promotion videos & 5. & Complete brand visual identity \\
6. & Instructional videos & 6. & Banner graphics \\
7. & Popular science documentaries & 7. & Official account cover images \\
8. & Music videos (MV) & 8. & Xiaohongshu covers \\
9. & Concert recordings & 9. & Marketing posters \\
10. & Variety shows & 10. & Avatar design \\
11. & Story videos & 11. & Static emoji generation \\
12. & Video podcasts & 12. & ICON design \\
13. & Picture books & 13. & LOGO design \\
14. & Dynamic comics & 14. & Mini-game UI design \\
15. & Animated short films & 15. & Character design \\
16. & Animated movies & 16. & Character action design \\
17. & Game adaptation films & 17. & Scene design \\
18. & Game videos & 18. & Storyboards \\
19. & Movies & 19. & Picture Book \\
20. & Short dramas & 20. & Stylization \\
21. & Story explanations & 21. & Realistic Photography \\
\bottomrule
\end{tabular}
\end{table}

\begin{table}[htbp]
\centering
\renewcommand{\tabcolsep}{1.1pt}
\renewcommand{\arraystretch}{1.1}
\caption{VisGenEnv integrates 36 visual creation tools.}
\label{tab:tools}
\begin{tabular}{lll}
\hline
\textbf{Tool Category} & \textbf{Tool Function} & \textbf{Tool Name} \\
\hline
\multirow{6}{*}{Text-to-Text}
& Storyboard Text Polishing (Claude) & tool\_prompt\_refine \\
& Storyboard Generation (Claude) & tool\_video\_shot\_gen \\
& Script Tool (Claude) & tool\_video\_script\_gen \\
& Storyboard Polishing (Claude) & tool\_storyboard\_polish \\
& Script \& Storyboard Polishing & tool\_script\_storyboard\_merge \\
& Text-to-Video (Veo3) & tool\_text2video\_veo \\
\hline
\multirow{5}{*}{Text-to-Image} 
& Text-to-Image (nano-banana) & tool\_text2image\_gemini \\
& Text-to-Image (hunyuan) & tool\_text2image\_hunyuan \\
& Text-to-Image (ByteDance) & tool\_text2image\_seed \\
& Text-to-Image (GPT) & tool\_text2image\_gpt \\
& Text-to-Image (Qwen) & tool\_text2image\_qwen \\
\hline
\multirow{3}{*}{Image-to-Image}
& Image-to-Image (nano-banana) & tool\_image\_edit\_gemini \\
& Image-to-Image (Qwen) & tool\_image\_edit\_qwen \\
& Image-to-Image (GPT) & tool\_image\_edit\_gpt \\
\hline
\multirow{2}{*}{Image-to-Video}
& Image-to-Video (Keling) & tool\_image2video\_keling \\
& Image-to-Video (Veo3) & tool\_image2video\_veo3 \\
\hline
\multirow{12}{*}{Audio Generation}
& Music Generation (Suno) & tool\_music\_suno \\
& Video Sound Effect Generation & tool\_sound\_fx\_gen \\
& TTS Generation & tool\_tts\_generation \\
& Video Composition (MoviePy) & tool\_video\_composite \\
& Video Clip - MoviePy Post-processing & tool\_video\_postprocess \\
& Video Generation Automation Pipeline & tool\_video\_auto\_pipeline \\
& Beat Detection Tool & tool\_beat\_detect \\
& Video Editing (Trim) & tool\_video\_trim\_edit \\
& Video Speed Change & tool\_video\_speed\_adjust \\
& TTS + Composition Tool & tool\_tts\_composite \\
& Audio Editing & tool\_audio\_edit\_cut \\
& Add Subtitles & tool\_subtitle\_add\_text \\
\hline
\multirow{1}{*}{Multimodal}
& Video Understanding (Gemini2.5-Pro) & tool\_video\_analysis \\
\multirow{1}{*}{Understanding} & Audio Understanding (Gemini2.5-Pro) & tool\_audio\_analysis \\
& Image Understanding (Gemini2.5-Pro) & tool\_image\_analysis \\
\hline
\multirow{5}{*}{Other}
& Tavily Search - Content Extraction & tool\_search\_content \\
& Inspiration Search & tool\_search\_inspire \\
& Summary Tool & tool\_content\_summary \\
& To-Do List & tool\_task\_manager \\
& HTML Generation Tool & tool\_html\_builder \\
\hline
\end{tabular}
\end{table}





\begin{figure}[ht]
\vspace{-3mm}
\centering
\includegraphics[width=1.0\textwidth]{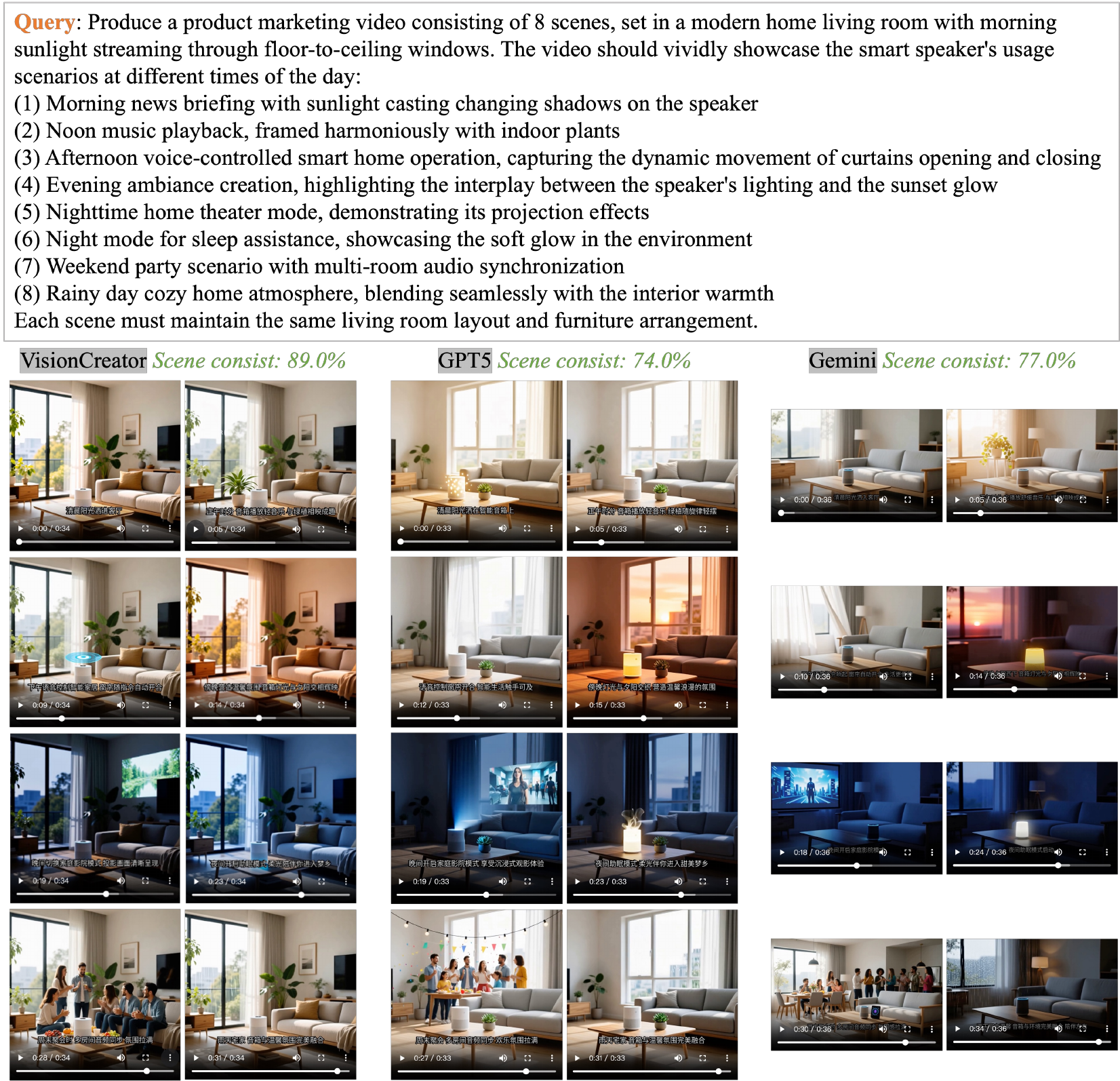}
\vspace{-5mm}
\caption{Visualizations for scene consistency.}
\vspace{-3mm}
\label{fig:vis_4}
\end{figure}


\begin{figure}[ht]
\vspace{-3mm}
\centering
\includegraphics[width=1.0\textwidth]{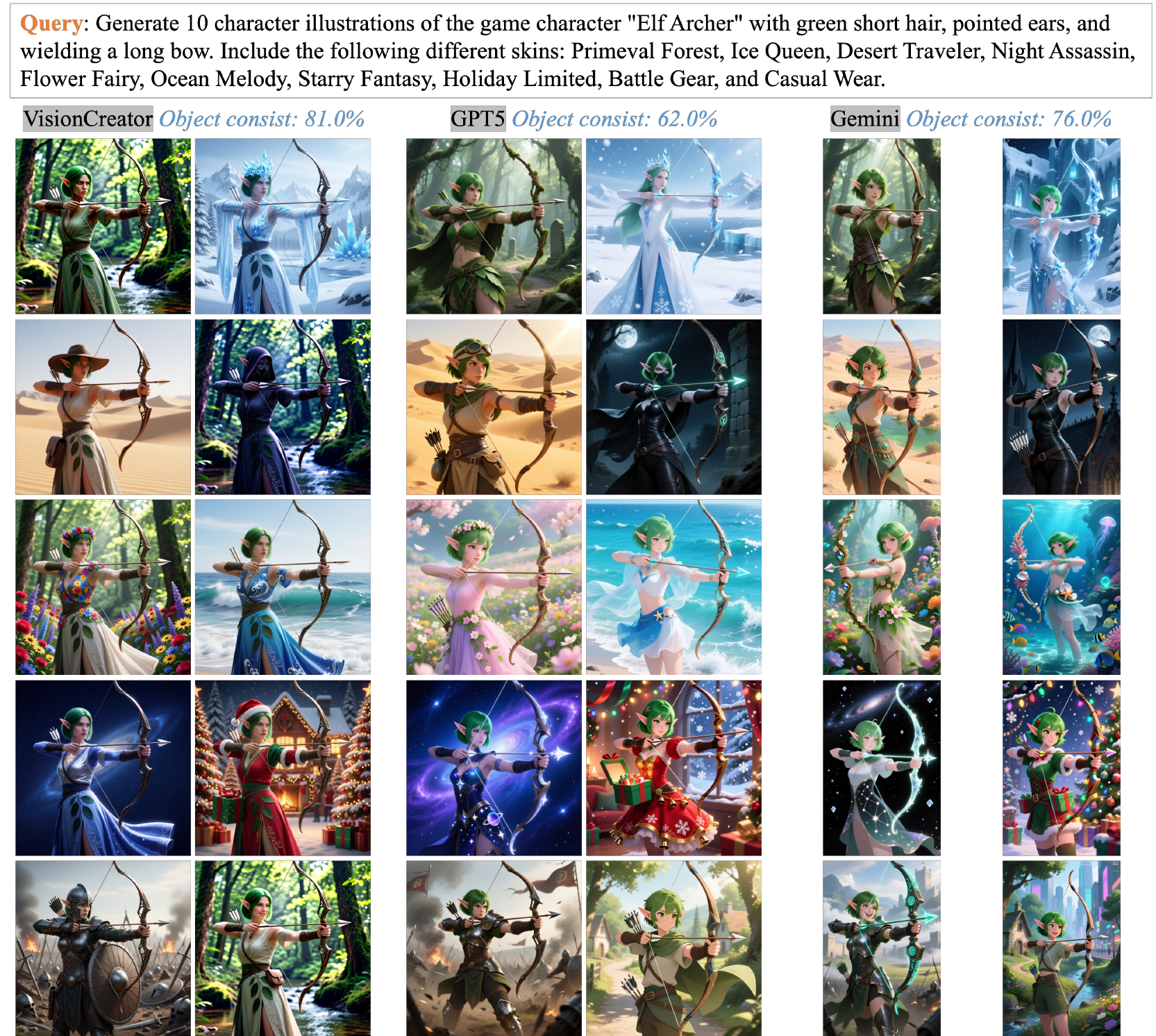}
\vspace{-5mm}
\caption{Visualizations for object consistency.}
\vspace{-3mm}
\label{fig:vis_6}
\end{figure}


\begin{figure}[ht]
\vspace{-3mm}
\centering
\includegraphics[width=1.0\textwidth]{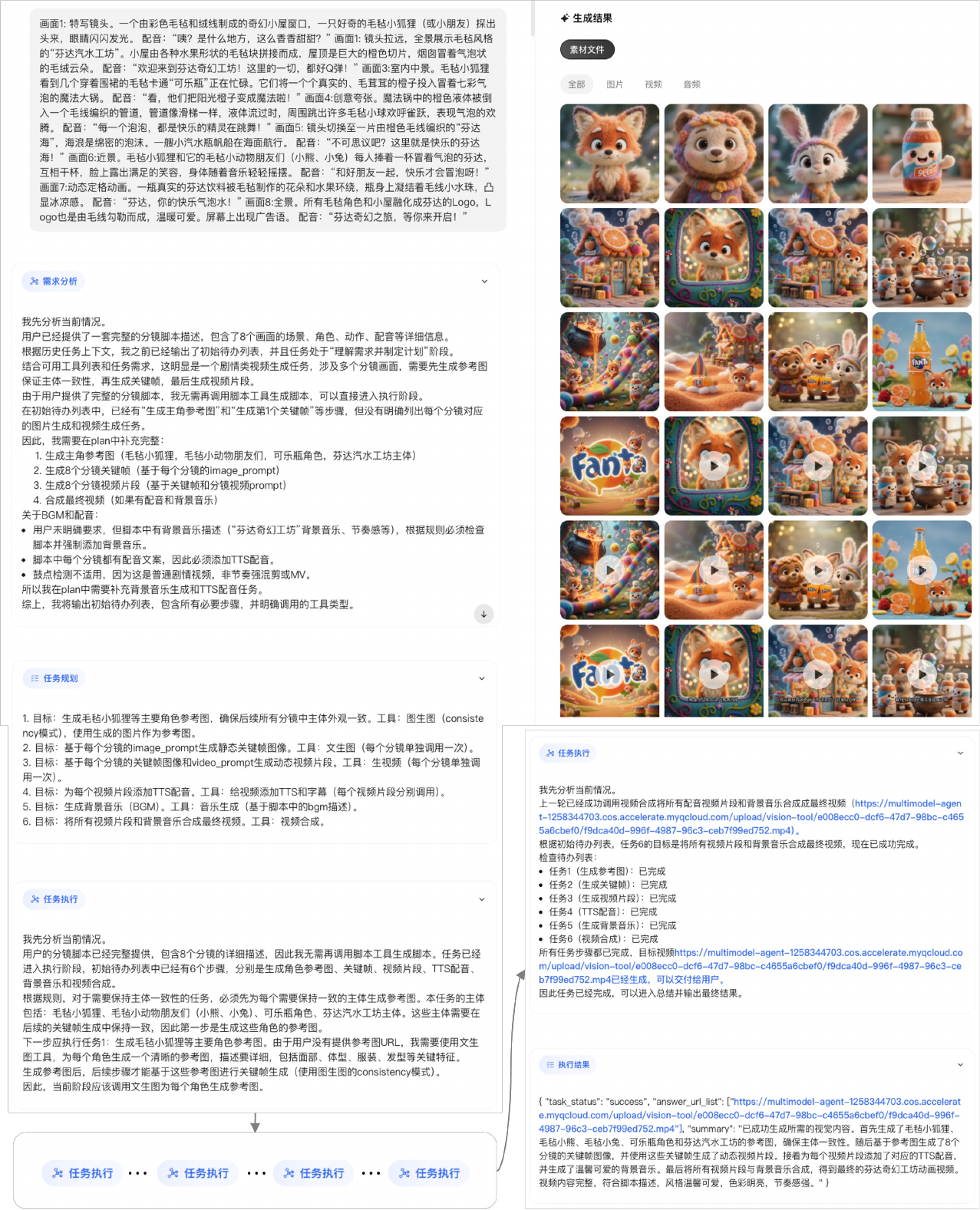}
\vspace{-5mm}
\caption{Demo page.}
\vspace{-3mm}
\label{fig:demo1}
\end{figure}

\clearpage

\clearpage

\bibliography{colm2024_conference}

\begin{thebibliography}{40}
\providecommand{\natexlab}[1]{#1}
\providecommand{\url}[1]{\texttt{#1}}
\expandafter\ifx\csname urlstyle\endcsname\relax
  \providecommand{\doi}[1]{doi: #1}\else
  \providecommand{\doi}{doi: \begingroup \urlstyle{rm}\Url}\fi

\bibitem[Blattmann et~al.(2023)Blattmann, Dockhorn, Kulal, Mendelevitch,
  Kilian, Lorenz, Levi, English, Voleti, Letts, et~al.]{blattmann2023stable}
Andreas Blattmann, Tim Dockhorn, Sumith Kulal, Daniel Mendelevitch, Maciej
  Kilian, Dominik Lorenz, Yam Levi, Zion English, Vikram Voleti, Adam Letts,
  et~al.
\newblock Stable video diffusion: Scaling latent video diffusion models to
  large datasets.
\newblock \emph{arXiv preprint arXiv:2311.15127}, 2023.

\bibitem[Chen et~al.(2020)Chen, Radford, Child, Wu, Jun, Luan, and
  Sutskever]{chen2020generative}
Mark Chen, Alec Radford, Rewon Child, Jeffrey Wu, Heewoo Jun, David Luan, and
  Ilya Sutskever.
\newblock Generative pretraining from pixels.
\newblock In \emph{International conference on machine learning}, pp.\
  1691--1703. PMLR, 2020.

\bibitem[comfyanonymous(2023)]{comfyanonymous2023comfyui}
comfyanonymous.
\newblock Comfyui.
\newblock \url{https://github.com/comfyanonymous/ComfyUI}, 2023.
\newblock GitHub repository.

\bibitem[Cui et~al.(2025)Cui, Chen, Deng, Huang, Li, Liu, Liu, Luo, Wang, Wang,
  et~al.]{cui2025emu3}
Yufeng Cui, Honghao Chen, Haoge Deng, Xu~Huang, Xinghang Li, Jirong Liu, Yang
  Liu, Zhuoyan Luo, Jinsheng Wang, Wenxuan Wang, et~al.
\newblock Emu3. 5: Native multimodal models are world learners.
\newblock \emph{arXiv preprint arXiv:2510.26583}, 2025.

\bibitem[Deng et~al.(2025)Deng, Zhu, Li, Gou, Li, Wang, Zhong, Yu, Nie, Song,
  et~al.]{deng2025emerging}
Chaorui Deng, Deyao Zhu, Kunchang Li, Chenhui Gou, Feng Li, Zeyu Wang, Shu
  Zhong, Weihao Yu, Xiaonan Nie, Ziang Song, et~al.
\newblock Emerging properties in unified multimodal pretraining.
\newblock \emph{arXiv preprint arXiv:2505.14683}, 2025.

\bibitem[Esser et~al.(2024)Esser, Kulal, Blattmann, Entezari, M{\"u}ller,
  Saini, Levi, Lorenz, Sauer, Boesel, et~al.]{esser2024scaling}
Patrick Esser, Sumith Kulal, Andreas Blattmann, Rahim Entezari, Jonas
  M{\"u}ller, Harry Saini, Yam Levi, Dominik Lorenz, Axel Sauer, Frederic
  Boesel, et~al.
\newblock Scaling rectified flow transformers for high-resolution image
  synthesis.
\newblock In \emph{Forty-first International Conference on Machine Learning},
  2024.

\bibitem[Fan et~al.(2024)Fan, Li, Qin, Li, Sun, Rubinstein, Sun, He, and
  Tian]{fan2024fluid}
Lijie Fan, Tianhong Li, Siyang Qin, Yuanzhen Li, Chen Sun, Michael Rubinstein,
  Deqing Sun, Kaiming He, and Yonglong Tian.
\newblock Fluid: Scaling autoregressive text-to-image generative models with
  continuous tokens.
\newblock \emph{arXiv preprint arXiv:2410.13863}, 2024.

\bibitem[{Glaive AI}(2023)]{function_calling_chatml_dataset}
{Glaive AI}.
\newblock Function-calling-chatml: Function calling dataset in chatml format.
\newblock
  \url{https://modelscope.cn/datasets/AI-ModelScope/function-calling-chatml},
  2023.

\bibitem[Guo et~al.(2025)Guo, Xu, Wang, Lin, Zhou, Zhang, Su, and
  Chen]{guo2025comfymind}
Litao Guo, Xinli Xu, Luozhou Wang, Jiantao Lin, Jinsong Zhou, Zixin Zhang,
  Bolan Su, and Ying-Cong Chen.
\newblock Comfymind: Toward general-purpose generation via tree-based planning
  and reactive feedback.
\newblock \emph{arXiv preprint arXiv:2505.17908}, 2025.

\bibitem[Han et~al.(2024)Han, Liu, Jiang, Yan, Zhang, Yuan, Peng, and
  Liu]{han2024infinity}
Jian Han, Jinlai Liu, Yi~Jiang, Bin Yan, Yuqi Zhang, Zehuan Yuan, Bingyue Peng,
  and Xiaobing Liu.
\newblock Infinity: Scaling bitwise autoregressive modeling for high-resolution
  image synthesis.
\newblock \emph{arXiv preprint arXiv:2412.04431}, 2024.

\bibitem[Ho et~al.(2020)Ho, Jain, and Abbeel]{ho2020denoising}
Jonathan Ho, Ajay Jain, and Pieter Abbeel.
\newblock Denoising diffusion probabilistic models.
\newblock \emph{Advances in neural information processing systems},
  33:\penalty0 6840--6851, 2020.

\bibitem[Huang et~al.(2025)Huang, Ma, Zhao, Wu, Ji, Zhang, Hu, Sun, and
  Ji]{huang2025comfygpt}
Oucheng Huang, Yuhang Ma, Zeng Zhao, Mingrui Wu, Jiayi Ji, Rongsheng Zhang,
  Zhipeng Hu, Xiaoshuai Sun, and Rongrong Ji.
\newblock Comfygpt: A self-optimizing multi-agent system for comprehensive
  comfyui workflow generation.
\newblock \emph{arXiv preprint arXiv:2503.17671}, 2025.

\bibitem[Labs(2024)]{flux2024}
Black~Forest Labs.
\newblock Flux.
\newblock \url{https://github.com/black-forest-labs/flux}, 2024.

\bibitem[Lai et~al.(2024)Lai, Zhang, Liu, Li, Lu, and Guo]{lai2024spider}
Jinxiang Lai, Jie Zhang, Jun Liu, Jian Li, Xiaocheng Lu, and Song Guo.
\newblock Spider: Any-to-many multimodal llm.
\newblock \emph{arXiv preprint arXiv:2411.09439}, 2024.

\bibitem[Li et~al.(2023)Li, Yin, Li, Chen, Wang, Ren, Li, Yang, Xu, Sun, Kong,
  and Liu]{m3it_dataset}
Lei Li, Yuwei Yin, Shicheng Li, Liang Chen, Peiyi Wang, Shuhuai Ren, Mukai Li,
  Yazheng Yang, Jingjing Xu, Xu~Sun, Lingpeng Kong, and Qi~Liu.
\newblock M$^3$it: A large-scale dataset towards multi-modal multilingual
  instruction tuning.
\newblock In \emph{arXiv preprint arXiv:2306.04387}, June 2023.
\newblock URL \url{https://arxiv.org/abs/2306.04387}.

\bibitem[Li et~al.(2025{\natexlab{a}})Li, Kallidromitis, Gokul, Liao, Kato,
  Kozuka, and Grover]{li2025omniflow}
Shufan Li, Konstantinos Kallidromitis, Akash Gokul, Zichun Liao, Yusuke Kato,
  Kazuki Kozuka, and Aditya Grover.
\newblock Omniflow: Any-to-any generation with multi-modal rectified flows.
\newblock In \emph{Proceedings of the Computer Vision and Pattern Recognition
  Conference}, pp.\  13178--13188, 2025{\natexlab{a}}.

\bibitem[Li et~al.(2025{\natexlab{b}})Li, Lin, Jiang, Cao, Liu, Zhang, Huang,
  Chen, Sun, Wang, Lu, Qin, Zhu, Yao, Fan, Li, Wang, Liu, Zhu, Zhu, Shi, Wang,
  Guan, Tang, Liu, Jiang, Yang, Liu, Zhang, and
  Zhou]{agent_foundation_models_repo}
Weizhen Li, Jianbo Lin, Zhuosong Jiang, Jingyi Cao, Xinpeng Liu, Jiayu Zhang,
  Zhenqiang Huang, Qianben Chen, Weichen Sun, Qiexiang Wang, Hongxuan Lu,
  Tianrui Qin, Chenghao Zhu, Yi~Yao, Shuying Fan, Xiaowan Li, Tiannan Wang, Pai
  Liu, King Zhu, He~Zhu, Dingfeng Shi, Piaohong Wang, Yeyi Guan, Xiangru Tang,
  Minghao Liu, Yuchen~Eleanor Jiang, Jian Yang, Jiaheng Liu, Ge~Zhang, and
  Wangchunshu Zhou.
\newblock Chain-of-agents: End-to-end agent foundation models via multi-agent
  distillation and agentic rl.
\newblock \emph{arXiv preprint arXiv:2508.13167}, August 2025{\natexlab{b}}.
\newblock URL \url{https://arxiv.org/abs/2508.13167}.

\bibitem[Lin et~al.(2024)Lin, Ge, Cheng, Li, Zhu, Wang, He, Ye, Yuan, Chen,
  et~al.]{lin2024open}
Bin Lin, Yunyang Ge, Xinhua Cheng, Zongjian Li, Bin Zhu, Shaodong Wang, Xianyi
  He, Yang Ye, Shenghai Yuan, Liuhan Chen, et~al.
\newblock Open-sora plan: Open-source large video generation model.
\newblock \emph{arXiv preprint arXiv:2412.00131}, 2024.

\bibitem[Liu(2025)]{deepseek_r1_distill_110k_sft}
Cong Liu.
\newblock Chinese-deepseek-r1-distill-data-110k-sft.
\newblock
  \url{https://modelscope.cn/datasets/liucong/Chinese-DeepSeek-R1-Distill-data-110k-SFT},
  2025.

\bibitem[Liu et~al.(2025)Liu, Huang, Zeng, Hao, Yu, Li, Wang, Gan, Liu, Yu,
  Wang, Wang, Ning, Hou, Wang, Wu, Wang, Liu, Wang, Tang, Tu, Shang, Jiang,
  Tang, Lian, Liu, and Chen]{toolace_dataset}
Weiwen Liu, Xu~Huang, Xingshan Zeng, Xinlong Hao, Shuai Yu, Dexun Li, Shuai
  Wang, Weinan Gan, Zhengying Liu, Yuanqing Yu, Zezhong Wang, Yuxian Wang,
  Wu~Ning, Yutai Hou, Bin Wang, Chuhan Wu, Xinzhi Wang, Yong Liu, Yasheng Wang,
  Duyu Tang, Dandan Tu, Lifeng Shang, Xin Jiang, Ruiming Tang, Defu Lian, Qun
  Liu, and Enhong Chen.
\newblock Toolace: Winning the points of llm function calling.
\newblock In \emph{International Conference on Learning Representations
  (ICLR)}, 2025.
\newblock URL \url{https://arxiv.org/abs/2409.00920}.

\bibitem[{ModelScope Team}(2024)]{msagent}
{ModelScope Team}.
\newblock Ms-agent: Modelscope agent sft dataset.
\newblock \url{https://modelscope.cn/datasets/iic/ms_agent}, 2024.

\bibitem[Pang et~al.(2024)Pang, Jin, Yang, Lin, Zhu, Tang, Chen, Tay, Lim,
  Yang, et~al.]{pang2024next}
Yatian Pang, Peng Jin, Shuo Yang, Bin Lin, Bin Zhu, Zhenyu Tang, Liuhan Chen,
  Francis~EH Tay, Ser-Nam Lim, Harry Yang, et~al.
\newblock Next patch prediction for autoregressive visual generation.
\newblock \emph{arXiv preprint arXiv:2412.15321}, 2024.

\bibitem[Peebles \& Xie(2023)Peebles and Xie]{peebles2023scalable}
William Peebles and Saining Xie.
\newblock Scalable diffusion models with transformers.
\newblock In \emph{Proceedings of the IEEE/CVF International Conference on
  Computer Vision}, pp.\  4195--4205, 2023.

\bibitem[Peng et~al.(2023{\natexlab{a}})Peng, Li, He, Galley, and
  Gao]{alpaca_gpt4_en}
Baolin Peng, Chunyuan Li, Pengcheng He, Michel Galley, and Jianfeng Gao.
\newblock Alpaca-gpt4-data-en: English instruction following dataset.
\newblock
  \url{https://modelscope.cn/datasets/AI-ModelScope/alpaca-gpt4-data-en},
  2023{\natexlab{a}}.

\bibitem[Peng et~al.(2023{\natexlab{b}})Peng, Li, He, Galley, and
  Gao]{alpaca_gpt4_zh}
Baolin Peng, Chunyuan Li, Pengcheng He, Michel Galley, and Jianfeng Gao.
\newblock Alpaca-gpt4-data-zh: Chinese instruction following dataset.
\newblock
  \url{https://modelscope.cn/datasets/AI-ModelScope/alpaca-gpt4-data-zh},
  2023{\natexlab{b}}.

\bibitem[{PowerInfer Team}(2025)]{longcot_refine_500k}
{PowerInfer Team}.
\newblock Longcot-refine-500k: Long chain-of-thought reasoning dataset.
\newblock \url{https://modelscope.cn/datasets/PowerInfer/LONGCOT-Refine-500K},
  2025.

\bibitem[Qin et~al.(2024)Qin, Liang, Ye, Zhu, Yan, Lu, Lin, Cong, Tang, Qian,
  Zhao, Hong, Tian, Xie, Zhou, Gerstein, Li, Liu, and Sun]{toolbench_dataset}
Yujia Qin, Shihao Liang, Yining Ye, Kunlun Zhu, Lan Yan, Yaxi Lu, Yankai Lin,
  Xin Cong, Xiangru Tang, Bill Qian, Sihan Zhao, Lauren Hong, Runchu Tian,
  Ruobing Xie, Jie Zhou, Mark Gerstein, Dahai Li, Zhiyuan Liu, and Maosong Sun.
\newblock Toolllm: Facilitating large language models to master 16000+
  real-world apis.
\newblock In \emph{International Conference on Learning Representations
  (ICLR)}, 2024.
\newblock URL \url{https://arxiv.org/abs/2307.16789}.

\bibitem[Ramesh et~al.(2022)Ramesh, Dhariwal, Nichol, Chu, and
  Chen]{ramesh2022hierarchical}
Aditya Ramesh, Prafulla Dhariwal, Alex Nichol, Casey Chu, and Mark Chen.
\newblock Hierarchical text-conditional image generation with clip latents.
\newblock \emph{arXiv preprint arXiv:2204.06125}, 1\penalty0 (2):\penalty0 3,
  2022.

\bibitem[Saharia et~al.(2022)Saharia, Chan, Saxena, Li, Whang, Denton,
  Ghasemipour, Gontijo~Lopes, Karagol~Ayan, Salimans,
  et~al.]{saharia2022photorealistic}
Chitwan Saharia, William Chan, Saurabh Saxena, Lala Li, Jay Whang, Emily~L
  Denton, Kamyar Ghasemipour, Raphael Gontijo~Lopes, Burcu Karagol~Ayan, Tim
  Salimans, et~al.
\newblock Photorealistic text-to-image diffusion models with deep language
  understanding.
\newblock \emph{Advances in neural information processing systems},
  35:\penalty0 36479--36494, 2022.

\bibitem[Singer et~al.(2022)Singer, Polyak, Hayes, Yin, An, Zhang, Hu, Yang,
  Ashual, Gafni, et~al.]{singer2022make}
Uriel Singer, Adam Polyak, Thomas Hayes, Xi~Yin, Jie An, Songyang Zhang, Qiyuan
  Hu, Harry Yang, Oron Ashual, Oran Gafni, et~al.
\newblock Make-a-video: Text-to-video generation without text-video data.
\newblock \emph{arXiv preprint arXiv:2209.14792}, 2022.

\bibitem[Sun et~al.(2024)Sun, Jiang, Chen, Zhang, Peng, Luo, and
  Yuan]{sun2024autoregressive}
Peize Sun, Yi~Jiang, Shoufa Chen, Shilong Zhang, Bingyue Peng, Ping Luo, and
  Zehuan Yuan.
\newblock Autoregressive model beats diffusion: Llama for scalable image
  generation.
\newblock \emph{arXiv preprint arXiv:2406.06525}, 2024.

\bibitem[Tian et~al.(2024)Tian, Jiang, Yuan, Peng, and Wang]{tian2024visual}
Keyu Tian, Yi~Jiang, Zehuan Yuan, Bingyue Peng, and Liwei Wang.
\newblock Visual autoregressive modeling: Scalable image generation via
  next-scale prediction.
\newblock \emph{arXiv preprint arXiv:2404.02905}, 2024.

\bibitem[Wu et~al.(2025)Wu, Zhu, and Shou]{wu2025automated}
Weijia Wu, Zeyu Zhu, and Mike~Zheng Shou.
\newblock Automated movie generation via multi-agent cot planning.
\newblock \emph{arXiv preprint arXiv:2503.07314}, 2025.

\bibitem[Xiao et~al.(2025)Xiao, Yang, Zhang, Cai, Zhao, Guo, Wetzstein,
  Agrawala, Yuille, and Jiang]{xiao2025captain}
Junfei Xiao, Ceyuan Yang, Lvmin Zhang, Shengqu Cai, Yang Zhao, Yuwei Guo,
  Gordon Wetzstein, Maneesh Agrawala, Alan Yuille, and Lu~Jiang.
\newblock Captain cinema: Towards short movie generation.
\newblock \emph{arXiv preprint arXiv:2507.18634}, 2025.

\bibitem[Xu et~al.(2025{\natexlab{a}})Xu, Mei, Li, Wu, Yan, Lai, Zhang, and
  Wu]{xu2025mm}
Xuenan Xu, Jiahao Mei, Chenliang Li, Yuning Wu, Ming Yan, Shaopeng Lai,
  Ji~Zhang, and Mengyue Wu.
\newblock Mm-storyagent: Immersive narrated storybook video generation with a
  multi-agent paradigm across text, image and audio.
\newblock \emph{arXiv preprint arXiv:2503.05242}, 2025{\natexlab{a}}.

\bibitem[Xu et~al.(2025{\natexlab{b}})Xu, Wang, Yang, Wang, Luo, Zhang, Hu, and
  Zhang]{xu2025comfyui}
Zhenran Xu, Yiyu Wang, Xue Yang, Longyue Wang, Weihua Luo, Kaifu Zhang, Baotian
  Hu, and Min Zhang.
\newblock Comfyui-r1: Exploring reasoning models for workflow generation.
\newblock \emph{arXiv preprint arXiv:2506.09790}, 2025{\natexlab{b}}.

\bibitem[Xue et~al.(2025)Xue, Lu, Huang, Wang, Ouyang, and
  Bai]{xue2025comfybench}
Xiangyuan Xue, Zeyu Lu, Di~Huang, Zidong Wang, Wanli Ouyang, and Lei Bai.
\newblock Comfybench: Benchmarking llm-based agents in comfyui for autonomously
  designing collaborative ai systems.
\newblock In \emph{Proceedings of the Computer Vision and Pattern Recognition
  Conference}, pp.\  24614--24624, 2025.

\bibitem[Yang et~al.(2024)Yang, Teng, Zheng, Ding, Huang, Xu, Yang, Hong,
  Zhang, Feng, et~al.]{yang2024cogvideox}
Zhuoyi Yang, Jiayan Teng, Wendi Zheng, Ming Ding, Shiyu Huang, Jiazheng Xu,
  Yuanming Yang, Wenyi Hong, Xiaohan Zhang, Guanyu Feng, et~al.
\newblock Cogvideox: Text-to-video diffusion models with an expert transformer.
\newblock \emph{arXiv preprint arXiv:2408.06072}, 2024.

\bibitem[Zhang et~al.(2024)Zhang, Lan, Zhu, Liu, Hoang, Kokane, Yao, Tan,
  Prabhakar, Chen, Liu, Feng, Awalgaonkar, Murthy, Hu, Chen, Xu, Niebles,
  Heinecke, Wang, Savarese, and Xiong]{xlam_function_calling_60k}
Jianguo Zhang, Tian Lan, Ming Zhu, Zuxin Liu, Thai Hoang, Shirley Kokane,
  Weiran Yao, Juntao Tan, Akshara Prabhakar, Haolin Chen, Zhiwei Liu, Yihao
  Feng, Tulika Awalgaonkar, Rithesh Murthy, Eric Hu, Zeyuan Chen, Ran Xu,
  Juan~Carlos Niebles, Shelby Heinecke, Huan Wang, Silvio Savarese, and Caiming
  Xiong.
\newblock xlam: A family of large action models to empower ai agent systems.
\newblock \emph{arXiv preprint arXiv:2409.03215}, September 2024.
\newblock URL \url{https://arxiv.org/abs/2409.03215}.

\bibitem[Zheng et~al.(2025)Zheng, Huang, Liu, Zou, He, Zhang, Gu, Zhang, He,
  Zheng, et~al.]{zheng2025vbench}
Dian Zheng, Ziqi Huang, Hongbo Liu, Kai Zou, Yinan He, Fan Zhang, Lulu Gu,
  Yuanhan Zhang, Jingwen He, Wei-Shi Zheng, et~al.
\newblock Vbench-2.0: Advancing video generation benchmark suite for intrinsic
  faithfulness.
\newblock \emph{arXiv preprint arXiv:2503.21755}, 2025.

\end{thebibliography}
\bibliographystyle{colm2024_conference}
\end{document}